%% file: ijrr_draft.tex
\documentclass[Afour,sageh,times]{sagej}

\usepackage{moreverb,url}

\usepackage[colorlinks,bookmarksopen,bookmarksnumbered,citecolor=red,urlcolor=red]{hyperref}
\usepackage{amsthm}
\usepackage{algorithm}
\usepackage{algpseudocode}

\usepackage{subcaption}
\usepackage{todonotes}
\let\labelindent\relax
\usepackage[inline]{enumitem}
\usepackage{xpatch}
\xpatchcmd{\algorithmic}{\itemsep\z@}{\itemsep=12ex plus2pt}{}{}
\usepackage{amsmath}
\usepackage{amsfonts}
\usepackage{amssymb}
\usepackage{amsmath}
\usepackage{graphicx}
\usepackage{float}
\usepackage{color,soul}
\usepackage{subcaption}
\usepackage{todonotes}
\usepackage[inline]{enumitem}
\newtheorem{theorem}{Theorem}
\newtheorem{lemma}{Lemma}
\newtheorem{proposition}{Proposition}
\newtheorem{definition}{Definition}
\newtheorem{remark}{Remark}

\usepackage{mathtools}
\usepackage{pdfpages}
\usepackage{float}

\def\volumeyear{2016}

\begin{document}

\runninghead{Fouad et al.}

\title{Energy Sufficiency in Unknown Environments via Control Barrier Functions}

\author{Hassan Fouad, Vivek Shankar Varadharajan, Giovanni Beltrame}

\affiliation{Polytechnique Montr\'eal, Canada\\
}

\corrauth{Hassan Fouad, MISTLab,
Polytechnique Montr\'eal,
Quebec,
Canada}

\email{hassan.fouad@polymtl.ca}

\begin{abstract}
  Maintaining energy sufficiency of a battery-powered robot system is a
  essential for long-term missions. This capability should be flexible enough to
  deal with different types of environment and a wide range of missions, while
  constantly guaranteeing that the robot does not run out of energy. In this
  work we present a framework based on Control Barrier Functions (CBFs) that
  provides an energy sufficiency layer that can be applied on top of any path
  planner and provides guarantees on the robot's energy consumption during mission
  execution. In practice, we smooth the output of a generic path planner using
  double sigmoid functions and then use CBFs to ensure energy sufficiency along
  the smoothed path, for robots described by single integrator and unicycle
  kinematics. We present results using a physics-based robot simulator, as well
  as with real robots with a full localization and mapping stack to show the
  validity of our approach.
\end{abstract}


\maketitle

\section{Introduction}
\label{sec:intro}
\input{introduction_2.tex}

\section{Related work}
\label{sec:related_work}
\input{related_work.tex}

\section{Background}
\label{sec:background}
\input{background_energy.tex}

\section{Energy sufficiency over a static Bezier path}
\label{sec:static}
\input{static_energy.tex}

\section{Energy sufficiency over a dynamic path}
\label{sec:dynamic}
\input{dynamic_energy.tex}

\section{Application to unicycle-type robots}
\label{sec:corner_power}
\input{corner_power_method.tex}

\section{Results \label{sec:results}}
\input{result_energy.tex}

\section{Conclusions \label{sec:conc}}
\input{conclusion.tex}

\section{Acknowledgements}
The authors would like to thank Koresh Khateri for his useful insights and fruitful discussions, as well as Sameh Darwish and Karthik Soma for their help with the experiments.

\section{Declaration of conflicting interests}
The Authors declare that there is no conflict of interest.
\section{Funding}
This work was supported by the National Science and Engineering Research Council of Canada [Discovery Grant number 2019-05165].

%\pagebreak

%\bibliographystyle{IEEEtran}
\bibliographystyle{SageH.bst}
%\bibliography{ijrr_refs}
\bibliography{refs_ral}

\end{document}

%% file: introduction_2.tex
Current advances in robotics and its applications play a key role in extending
human abilities and allowing humans to handle arduous workloads and deal with
dangerous and uncertain environments. For instance, search and rescue
missions~\citep{balta2017integrated}, construction~\citep{yang2021hallway}, and
mining~\citep{thrun2004autonomous} put a strain on the human body as well as
being inherently dangerous. Moreover, tasks with a high degree of uncertainty like
terrestrial~\citep{best2022resilient} and
extraterrestrial~\citep{bajracharya2008autonomy} exploration benefit
immensely from using robots, especially with the current quest for planetary
exploration and the need to discover locations to host
humans~\citep{cushing2012candidate,titus2021science}.

To this end, endowing robots with the ability to recharge during a mission is of
vital importance to enable long term autonomy and successful execution of
missions over extended periods of time. This gives rise to a crucial need for
methods that guarantee that no robot runs out of energy mid-mission, i.e. energy
sufficiency, while at the same time having the needed flexibility to adapt to
various types of missions and environments. Many methods exist in literature to
achieve this goal: using static charging
stations~\citep{notomista2018persistification,ravankar2021multi,liu2014optimal,fouad2022energy},
using moving charging stations that do rendezvous with the robots during their
mission~\citep{mathew2015multirobot,kundu2021mobile,kamra2017combinatorial} or
deposit full batteries along robot's mission path~\citep{ding2019decentralized}.
The main shortcomings of these methods are that they either do not provide formal
guarantees on performance, or they have limited ability to deal with scenarios
involving unstructured and uncertain environments, e.g. in exploration missions
where maps are not known beforehand.

\begin{figure}[!htb]
	\centering
	\includegraphics[width=\columnwidth]{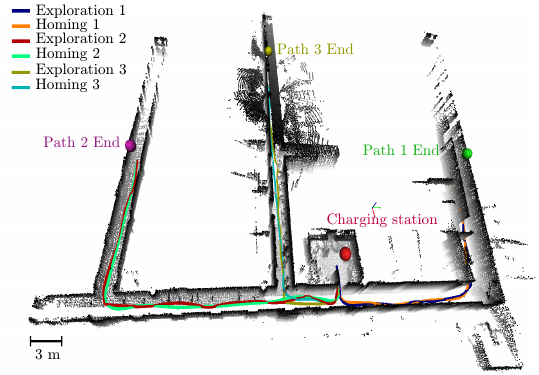}
	\caption{Maintaining energy sufficiency during the exploration of a corridor
    environment.}
	\label{fig:rover_explore}
\end{figure}

One way to tackle the issue of unstructured environments in light of energy
sufficiency is to perform path planning that incorporates energy cost as one of
its metrics. As an examples of this energy-aware path planning,
\citealt{alizadeh_optimized_2014,fu_targeted_2019,schneider2014electric}
formulate energy sufficiency as a combinatorial optimization problem with the
environment modelled as a weighted graph encoding energy costs, travel times,
and distances. One issue with these methods is the rapid increase in
computational complexity for large environments. Other methods emerged to deal
with this issue with heuristics like Genetic
Algorithms~\citep[GA,][]{li2018planning},
Tabu-Search~\citep[TS,][]{wang_staying-alive_2008} and Monte Carlo Tree
Search~\citep[MCTS][]{warsame2020energy}. However, one fundamental problem with
these methods mentioned so far is their need to know the map beforehand, which
may not be available for missions with unknown or dynamic environments such as
exploration tasks.

Tackling the issue of unknown and unstructured environments calls for the use of
exploration planners. Such planners use the collected sensor information over
time and provide two types of trajectories: exploration paths that maximize
environmental coverage, and homing paths from robot's current position to any
desired point in the map that is being incrementally built as the robot keeps
exploring. Several well-designed exploration planners exist in literature, many
of which developed within the scope of the DARPA SubTerranian
Challenge~\citep{subt_challenge}: the Graph-Based exploration planner
\citep[GBPlanner,][]{dang2019graph}, the Next-Best-View planner
\citep{bircher2016receding}, the motion primitives-based planner
\citep[MbPlanner,][]{dharmadhikari2020motion}), the Dual-Stage Viewpoint Planner
\citep{zhu2021dsvp}, and the TARE planner~\citep{cao2021tare}.

In this work, we present a modular and mission-agnostic framework that uses a
Control Barrier Function (CBF)~\citep{ames2019control} to guarantee energy sufficiency
when applied alongside an arbitrary exploration planner. The approach builds
upon our previous work \citep{fouad2020energy,fouad2022energy}, which provides
energy-sufficiency guarantees for robots in obstacle-free environments. We thus
leverage the ability of an exploration planner to deal with unstructured and
unknown environments and extend previous formulations to validate guarantees on
energy sufficiency over paths generated by this planner, allowing for more
realistic mission execution. The modular nature of our framework makes it
suitable for a wide range of applications that employ a path planner, especially
the exploration of unknown subterranean environments~\citep{dang2019graph}, ans
well a navigation in urban~\citep{mehta2015human,ramana2016motion,fu2015path},
and indoor environments~\citep{zhao2013path}.

In essence, the contribution of this paper is a CBF-based mission-agnostic
modular framework that can be applied in conjunction with any path planner to
ensure energy sufficiency of a robot in unknown and unstructured environment.
The framework applies to robots modelled as single integrator points or using
unicycle kinematics. The framework is validated through physics-based simulation
and on a physical AgileX Scout Mini rover, with a detailed description of our
hardware setup and software stack (which is also available as open-source).

The paper is organized as follows: Section ~\ref{sec:related_work} reviews the
literature around energy sufficiency, energy awareness in path planning, and some
relevant topics to our frameworks like path smoothing and control barrier
functions (CBFs); Section~\ref{sec:background} presents some preliminaries,
followed by the problem statement we are addressing; in Section~\ref{sec:static}
we lay out the main building blocks of our framework by addressing a case in
which a robot, modelled as a single integrator point, is stationary with
a non-changing path; in Section~\ref{sec:dynamic} we extend the results of the
previous Section to the case of a moving robot with a varying path due to
robot's motion and path updates; we then present a method for applying our
proposed framework with robots described by unicycle kinematics in
Section~\ref{sec:corner_power}; Section~\ref{sec:results} shows simulation and
hardware results; then we conclude the paper and provide a discussion along with
future work in Section~\ref{sec:conc}.
%%% Local Variables:
%%% mode: latex
%%% TeX-master: "ijrr_draft"
%%% End:

%% file: related_work.tex
Path planning methods for autonomous robots have been an active area of research
for a long time \citep{souissi2013path,patle2019review}. Different families of
path planning methods can be found in literature that vary in their purpose
(e.g., local planning vs. global planning), the way they encode the environment
(grid maps, visibility graphs, voronoi diagrams...), the type of systems
they plan for (e.g., holonomic, non holonomic, kinodynamic) and the way the path
is created (sampling the space, graph searching, potential fields...).

Endowing path planning with energy awareness has been treated in literature in
different forms that vary by purpose. For example, some works find energy
efficient paths within an environment so as to increase a mission's life span as
presented by \citet{jaroszek2014model} for four wheeled robots and by
\citet{gruning2020energy} for robots in hilly terrains. Other works focus on
ensuring robot's ability to carry out missions within certain energy capacity
and return to a charging station. For example, \citet{wang_staying-alive_2008}
use a graph with nodes representing tasks with energy costs and edges indicating
spatial connectivity with distances, then use Tabu search to solve a Traveling
Salesperson Problem (TSP) on this graph to minimize cost.
\citet{warsame2020energy} use a probabilistic roadmap method to generate a graph
with routes to different goals and charging nodes, then they use Monte Carlo
Tree Search (MCTS) to create a tour vising all goals, while having a utility
function that diverts the robot from its tour to recharge when needed.
\citet{hao2021automatic} proposes something similar by creating an idealized
version of the environment in the form of a MAKLINK graph, then use the Dijkstra's
algorithm for finding paths to charging station. \citet{li2018planning} consider
the problem of UAV coverage of an area while needing to recharge, and uses a mix
of grid maps and genetic algorithms (GA) to produce trajectories that minimize
mission time and cost while penalizing energy loss. In the electric vehicle
literature, similar graph representations of environment are typically used, and
an optimization problem is solved over the graph. \citet{schneider2014electric}
formulate the problem as a variation of the vehicle routing problem and use
mixed integer programming to find the optimal paths. Similarly,
\citet{fu_targeted_2019} formulate an integer program that aims to find the
best path with least cost to go from destination to goal while charging at a
station. The problem is then solved in two stages: building a meta graph
of best paths from destination to goal passing through stations, and the then
using Dijkstra's algorithm to find the best path of this meta graph. \emph{It is
  worth noting that these methods typically do not provide performance guarantees.}

The output of sampling based path planners is often in the form of waypoints.
There is often a need to smooth the resulting piecewise linear paths to reduce
the effect of sharp turns, which gives rise to a significant body of work
pertaining to path smoothing \citep{ravankar2018path}: using Bezier curves
\citep{cimurs2017bezier,simba2016real}, B-splines \citep{noreen2020collision},
among many others. \citet{cimurs2017bezier} provide a framework for
interpolating a set of waypoints with cubic Bezier segments in a way that
maintains curvature limits and ensures no collision between obstacles and the
interpolated path. \citet{noreen2020collision} uses clamped B-splines to produce
$C^2$ continuous paths, and they provide a scheme for point insertion in
segments where there is collision with obstacles to iteratively rebuild the path
till no collision takes place.

Another body of work attempts to merge path planning and smoothing: for example,
\citet{elhoseny2018bezier} propose a method for finding shortest Bezier paths in
a cluttered environment, where Bezier control points are searched for to
minimize path length using Genetic Algorithms (GA). \citet{satai2021bezier}
provide a method for smoothing the output of variants of an $A^*$ planner by
considering the waypoints as Bezier curve control points and then introducing
insertion points between every two of these control points, then use quadratic
Bezier segments with the inserted points as control points to produce a smooth
path. \citet{wu2014bezier} describe a method for creating smooth paths in robot
soccer, where the authors use a 4th-order Bezier curve with control points
comprised of the robot's position and goal as ends, and the other robots'
positions as rest of control points to produce a dynamically changing and smooth
path.

We use Control Barrier Functions \citep[CBFs,][]{ames2019control} as the base of
our framework. Barrier functions have been used in optimization problems to
penalize solutions in unwanted regions of the solution space
\citep{forsgren2002interior}. This concept has been later exploited to certify
the safety of nonlinear systems \citep{prajna2004safety}, in the sense that
finding such functions guarantees a system's state does not to wander to unsafe
regions of the state space. The notion of Control Barrier Function was
introduced by \citet{wieland2007constructive} to express values of a system's
control input that ensures safety for a control affine system, and
\citet{ames2014control} introduced the popular method of using quadratic
programs to merge system tracking, encoded by a desired system input, and the
safe control input dictated by CBF constraints. Other methods use Control
Lyapunov Barrier Functions (CLBF) \citep{romdlony2014uniting} to achieve
tracking and safety simultaneously.

%%% Local Variables:
%%% mode: latex
%%% TeX-master: "ijrr_draft"
%%% End:

%% file: background_energy.tex
\subsection{Control barrier functions}
A control barrier function is a tool that has gained much attention lately as a
way of enforcing set forward invariance to achieve safety in control affine
systems of the form
\begin{equation*}
\dot{\mathbf{x}} = f(\mathbf{x})+g(\mathbf{x})u
\end{equation*}
where $u\in U\subset \mathbb{R}^m$ is the input, $U$ is the set of admissible
control inputs, $\mathbf{x}\in\mathbb{R}^n$ is the state of the system, and $f$
and $g$ are both Lipschitz continuous. In this context, what is meant by safety
is achieving set forward invariance of some safe set $\mathcal{C}$, meaning that
if the states start in $\mathcal{C}$ at $t=t_0$, they stay within $\mathcal{C}$
for all $t>t_0$. This safe set $\mathcal{C}$ is defined as the superlevel set of
a continuously differentiable function $h(x)$ in the following manner
\citep{ames2019control}:
\begin{equation}
\begin{split}
\mathcal{C} &= \{x\in\mathbb{R}^n:h(x)\geq 0\}\\
\partial\mathcal{C} &= \{x\in\mathbb{R}^n:h(x)=0\}\\
Int(\mathcal{C}) &= \{x\in\mathbb{R}^n:h(x)>0\}.
\end{split}
\end{equation}
This condition can be achieved by finding a value of control input that
satisfies $\dot{h}\geq -\alpha(h)$, with $\alpha(h)$ being an extended class
$\mathcal{K}$ function \citep{khalil2002nonlinear}.

\begin{definition}{\citep{ames2019control}}
	For a subset $\mathcal{W}\subset \mathcal{C}$, a continuously differentiable
  function $h(x)$ is said to be a zeroing control barrier function (ZCBF) if
  there exists a function $\alpha(h)$ s.t.
	\begin{equation}
	\sup_{u\in U}L_fh+L_ghu\geq -\alpha(h), \quad \forall x\in \mathcal{W}
	\end{equation} 
	where $L_fh$ and $L_gh$ are the Lie derivatives of $h(x)$ in direction of $f$
  and $g$ respectively.
\end{definition}

Supposing that we define the set of all safe inputs $U_{s} = \{u\in U:
L_fh+L_ghu\geq -\alpha(h)\}$, then any Lipschitz continuous controller $u\in
U_s$ guarantees that $\mathcal{C}$ is forward invariant \citep{ames2019control}.
Since the nominal control input $u_{nom}\in U$ for a mission may not belong to
$U_s$, there should be a way to enforce safety over the nominal mission input.
This could be done by the following quadratic program
(QP)~\citep{ames2019control}

\begin{equation}
\begin{aligned}
&  u^*=\underset{u}{\text{min}}
& & ||u-u_{nom}||^2 \\
& \quad \quad \quad \text{s.t.}
& & L_fh(x)+L_ghu \geq -\alpha(h) \\
%&&& X \succeq 0.
\end{aligned}
\label{eqn:QP}
\end{equation}
noting that $u^*$ tries to minimize the difference from $u_nom$, as long as
safety constraints are not violated.

\subsection{Problem definition} 
We adopt single integrator dynamics to describe the robot's position in 2D.
Moreover, we consider the energy consumed by the robot as the integration of its
consumed power, which in turn is a function of the robot's velocity
\begin{equation}
\begin{split}
\dot{x} &= u\\
\dot{E} &= \mathcal{P}(u)
\end{split}
\label{eqn:robot_dynamics}
\end{equation}
with $x\in\mathbb{R}^2$ being the robot's position ,
$u\in\mathcal{U}\in\mathbb{R}^2$ is the robot's velocity control action, $E > 0$
is the energy consumed and $\mathcal{P}(u) > 0$ is the power consumed by the
robot as a function of its input velocity. The power consumption follows the
following parabolic relation
\begin{equation}
\mathcal{P}(u) = m_0 + m_1 ||u|| + m_2||u||^2
\label{eqn:power_parabolic}
\end{equation}
for $m_0,m_1,m_2 > 0$. We consider a charging station at $x_c\in\mathbb{R}^2$
and that the robot starts a fast charge or a battery swap sequence as soon as it
is at a distance $\delta$ away from $x_c$, i.e. $||x-x_c||\leq \delta$.
%We use a linear model fitted from the power consumption of our experimental platform so that 
%\begin{equation}
%\mathcal{P}(u) = m_0 + m_1||u||
%\label{eqn:power_relation}
%\end{equation} 
%for $m_0,m_1 > 0$ 

Assume that there exists a path between a robot and a charging station described
by a set of waypoints $\mathcal{W}= \{w_1,w_2,\dots,w_{n_w}\} $, with
$w_i\in\mathbb{R}^{2}$ and the charging station at $w_{n_w}$, produced by a path
planner every $\mathcal{T}$ seconds. Provided that such robot is carrying out a
mission encoded by a desired control action $u_{d}$ and a nominal energy budget
$E_{nom}$, our objective is to ensure energy sufficiency for this robot, i.e.
$E_{nom} - E(t) \geq 0 \quad \forall t > t_0$, while taking into account the
path defined by $\mathcal{W}$ back to the charging station. Such scenario is
relevant in cases where ground robots are doing missions in complex or unknown
environments, or for flying robots in areas with no fly zones.

In this work we assume
%\begin{enumerate*} 
%\item The power consumption is uniform over the environment and depends only on
%  the input velocity.
%\item
  that the environment is static, i.e. obstacles don't change their positions
  during the mission.
%\item The paths produced by the path planner are two-way traversable, with the
%  power consumed not depending on the direction of moving along the
%  path.
%  \footnote{Examples of this assumption include ground robots moving on a flat
%    surface or a drones flying in still air with no wind fields.}
%\end{enumerate*}
%%% Local Variables:
%%% mode: latex
%%% TeX-master: "ijrr_draft"
%%% End:

%% file: static_energy.tex
%To introduce our approach for achieving energy sufficiency, we consider a case in which a robot is stationary at one end of the path, and the waypoints $\mathcal{W}$ defining it do not change. 

In this section we discuss the foundational ideas of our approach. We start by
considering a static scenario where the path does not change, and the robot is stationary and lies
at one end of the path, while the charging station lies on the other end.

Briefly, we construct a continuous parametric representation of the piecewise
linear path described by waypoints $\mathcal{W}$. We define a reference point
along the path that depends on a path parameter value, then we modify the
energy sufficiency framework in \citep{fouad2022energy} to manipulate the
location of the reference point in a manner proportional to available energy,
and we make the robot follow this reference point. This way we can generalize
the method in \citep{fouad2022energy} to environments with obstacles.

\subsection{Smooth path construction}
To ensure that the CBFs we are using are Lipschitz continuous, we use a smooth
parametric description of the piecewise linear path we receive from a path
planner as a set of waypoints $\mathcal{W}$. 

We define $p(s)$ to be a point on the path that corresponds to a parameter
$s\in[0,1]$, such that $p(0) = w_1$ and $p(1) = w_n$, i.e. $s=0$ at the
beginning of the path and $s = 1$ at its end. Such concept is common for
describing parametric splines like Bezier curves. We seek an expression for
$p(s)$ that closely follows a given piecewise linear path with waypoints
$\mathcal{W}$.

For some point $p$ lying on the path, we define the path parameter $s$ as being
the ratio of the path length from $w_1$ to $p$ to the total path length.
Figure~\ref{fig:pw_path} shows an illustrative example of five waypoints. We adopt a smooth representation for $p(s)$ using double sigmoid activation functions as follows
\begin{equation}
p(s) = \sum_{i=1}^{n_w-1}\sigma_{i}(s)\bar{w}_i(s)
\label{eqn:path_point}
\end{equation}
where $\bar{w}_i$ is expressed as 
\begin{equation}
\bar{w}_i(s) = \tfrac{s_{i+1}-s}{s_{i+1}-s_i} w_i + \tfrac{s-s_i}{s_{i+1}-s_i}w_{i+1}
\end{equation}
Here $s_i = \tfrac{L_i}{L}$ where $L_i = \sum_{k = 1
}^{i-1}||w_{k+1}-w_{k}||$ and $L = \sum_{k = 1 }^{n_w-1}||w_{k+1}-w_{k}||$. We
note that the relation between the path length $l(s)$ from path start (at $w_1$) to
point $p(s)$ is $l(s) = Ls$ (by definition of $s$).
\begin{figure}[!htb]
	\centering
	\fontsize{20pt}{20pt}\selectfont% or whatever fontsize you like
	\def\svgwidth{5.333in}
	\scalebox{0.5}{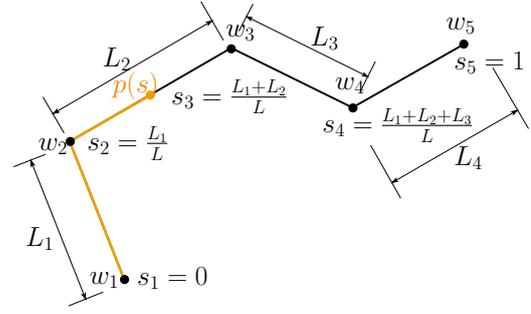}
	\caption{Am illustrative example of a path consisting of five waypoints. For a point $p(s)$ on the path $s$ is defined to be the ratio of the length of the orange segment to the total path length. In this illustration $L = \sum_{k=1}^{n_w-1}L_k$}
	\label{fig:pw_path}
\end{figure}

In \eqref{eqn:path_point}, $\sigma_i(s)$ is a double sigmoid function defined as 
\begin{equation}
\begin{split}
\sigma_i(s) &= \sigma_{i}^{r}(s)\sigma_{i}^{f}(s)\\ \sigma_{i}^{r}(s) &= \frac{1}{1+e^{-\beta(s-(s_i-\epsilon_1))}}\\
\sigma_{k}^{f}(s)&= \frac{1}{1+e^{\beta(s-(s_{i+1}+\epsilon_2))}}
\\\epsilon_1 &= \begin{cases}
\epsilon, \quad i = 1\\
0, \quad \text{otherwise}
\end{cases}\\ 
\epsilon_2 &= \begin{cases}
\epsilon, \quad i = n_w-1\\
0, \quad \text{otherwise}
\end{cases}
\end{split}
\label{eqn:sigmoids}
\end{equation}
where $\epsilon > 0$ and the superscripts $r$ and $f$ denote rising and falling
edges. The introduction of $\epsilon_1$ and $\epsilon_2$ to the first and last
segments in the previous relations is to emphasize that $\sigma_1(0) = 1$ and
$\sigma_{n_w-1}(1) = 1$, thus ensuring that $p(0) = w_1$ and $p(1) = w_{n_w}$,
otherwise $p(0) = p(1) \approx 0$ which is against the definition of $p(s)$.
This idea is illustrated in Figure~\ref{fig:sigmoids}. We also note that in any
transition region around $s=s_i$ there are two double sigmoid functions
involving $s_i$, namely $\sigma_{i-1}(s)$ and $\sigma_i(s)$. Furthermore, the
summation of these functions in the local neighbourhood of $s = s_i$ is equal to
one, which follows directly from adding $\sigma_{i-1}^{f}(s)$ and
$\sigma_{i}^{r}(s)$
\begin{equation}
\sigma_{i-1}^{f}+\sigma_{i}^{r}=\frac{2+e^{\beta(s-s_i)}+e^{-\beta(s-s_i)}}{2+e^{\beta(s-s_i)}+e^{-\beta(s-s_i)}} = 1
\label{eqn:trans_sigs}
\end{equation}
This idea is highlighted in Figure~\ref{fig:sigmoids}. The derivative $\tfrac{\partial p}{\partial s}$ is 
\begin{equation}
\begin{split}
\cfrac{\partial p}{\partial s} = \sum_{i=1}^{n_w-1}&\left(\sigma_{i}(s)\left(\frac{w_{i+1}-w_i}{s_{i+1}-s_i}\right)\right. \\&\left.+ \beta\sigma_i(s)(\sigma_{i}^{f}(s)-\sigma_{i}^{r}(s))\bar{w}_i(s)\right)
\end{split}
\label{eqn:diff_sigmoid}
\end{equation} 
\begin{figure}[!htb]
	\centering
	\includegraphics[trim=80 120 90 150,clip,width=\columnwidth]{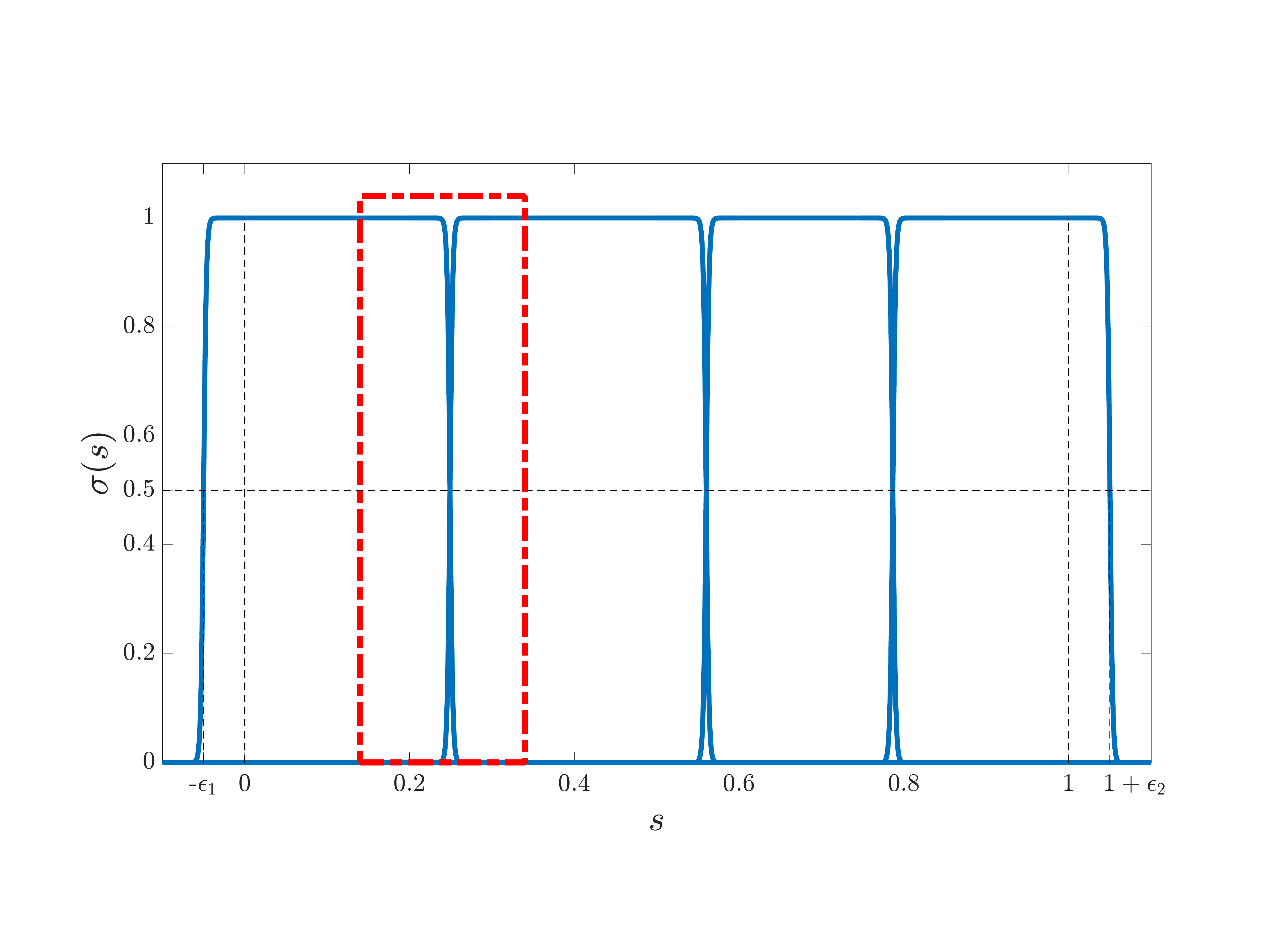}
	\caption{Example of double sigmoid functions $\sigma_k(s)$ for the set of five waypoints shown in Figure~\ref{fig:pw_path}. The use of $\epsilon_1$ and $\epsilon_2$ the way described in \eqref{eqn:sigmoids} leads to $\sigma_{1}^{r}(-\epsilon_1) = 0.5$ and $\sigma_{n_w}^{f}(1+\epsilon_2) = 0.5$, thus ensuring that $\sigma_1(0) = 1$ and $\sigma_{n_w-1}(1) = 1$. The red rectangle highlights a transition region, and it can be shown that the sum of the two sigmoids involved in this transition is equal to one.}
	\label{fig:sigmoids}
\end{figure}
We note that the larger the value of $\beta$ in \eqref{eqn:sigmoids} is the more
closely the smooth path described by \eqref{eqn:path_point} follows the
piecewise linear path between waypoints in $\mathcal{W}$.
Figure~\ref{fig:smooth_paths} shows examples of paths at different values of
$\beta$ for the same path depicted in Figure~\ref{fig:pw_path}.

\begin{figure}[!htb]
	\centering
	\includegraphics[trim=99 120 120 120,clip,width=\linewidth]{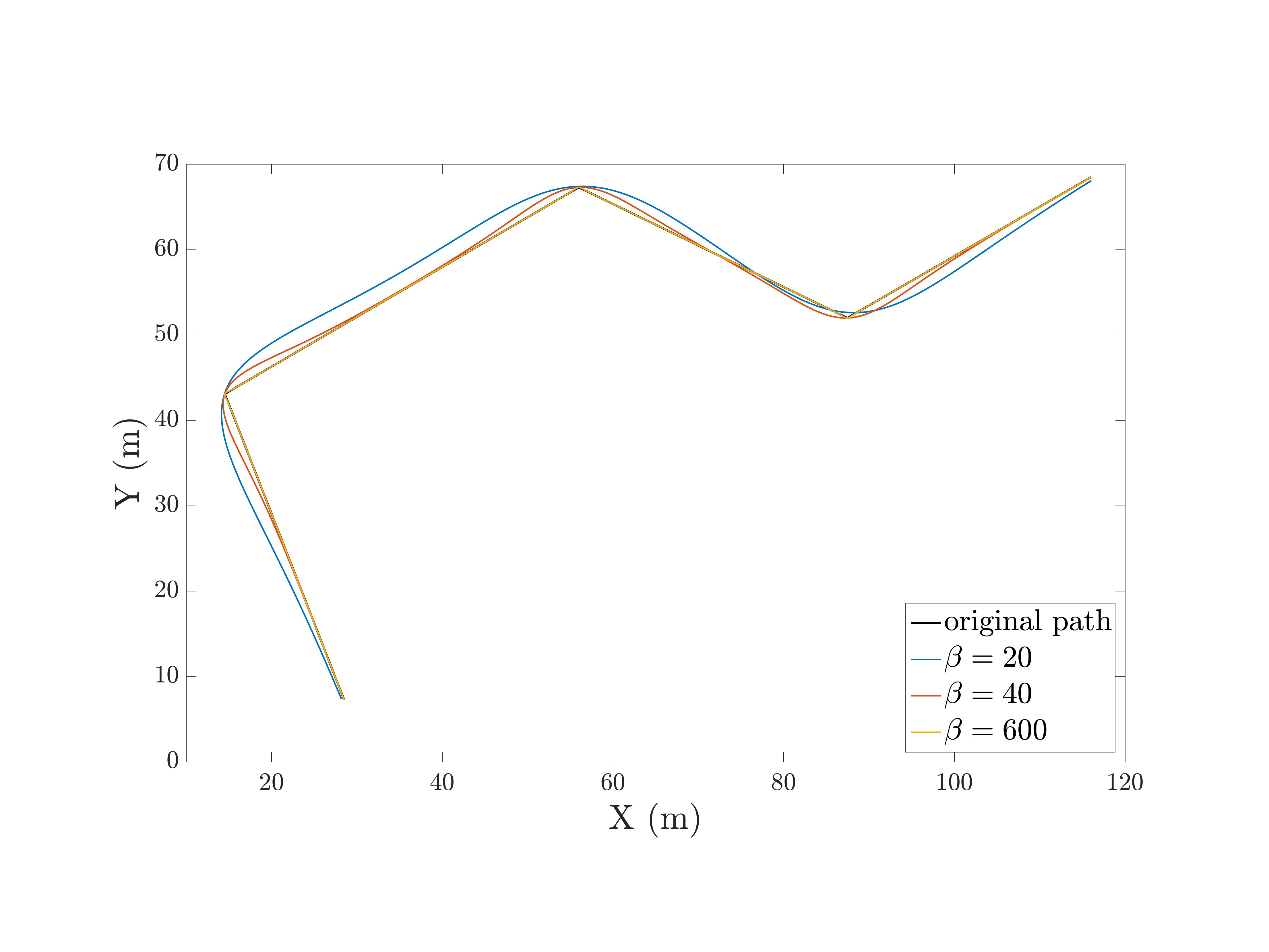}
	\caption{Demonstration of the effect of changing the value of $\beta$ in \eqref{eqn:sigmoids} on how closely \eqref{eqn:path_point} follows the original piecewise linear path.}
	\label{fig:smooth_paths}
\end{figure}

%The next lemma concerns the last expression in the summation \eqref{eqn:diff_sigmoid}.
\begin{lemma}\label{lemma:diff_sum}
	For a path described by $p(s)$ in \eqref{eqn:path_point}, with the double
  sigmoid functions as described in \eqref{eqn:sigmoids} and provided that
  $\beta \gg 1$ then the following statement holds
	\begin{equation*}
	\sum_{i=1}^{n_w-1}\Sigma_i(s)\bar{w}_i(s)\approx 0
	\end{equation*}
	where $\Sigma_i(s) = \beta\sigma_i(s)(\sigma_{i}^{f}(s)-\sigma_{i}^{r}(s))$
\end{lemma}
\begin{proof}
	At a waypoint $w_i$ we consider the two functions $\Sigma_{i-1}(s)$ and
  $\Sigma_i(s)$ (both involve $s=s_i$ in their definition) and we note that the
  values of other $\Sigma$ are equal to zero by definition. We want to evaluate
  $D=\Sigma_{i-1}(s)\bar{w}_{i-1}+\Sigma_{i}(s)\bar{w}_{i}$ but we do so within
  a band $\delta_s$ around $s_i$, i.e. at $s^\prime = s + \delta_s$
	\begin{equation}
	\begin{split}
	D &= \beta\sigma_{i-1}(s_i+\delta_s)(\sigma_{i-1}^{f}(s^\prime)-\sigma_{i-1}^{r}(s^\prime))\bar{w}_{i-1}(s^\prime)\\
	&+\beta\sigma_{i}(s_i+\delta_s)(\sigma_{i-1}^{f}(s^\prime)-\sigma_{i}^{r}(s^\prime))\bar{w}_{i}(s^\prime)
	\end{split}
	\end{equation}
	then substituting $\beta\gg 1$ in the last equation we get the following
	\begin{equation}
	D=\frac{\beta\delta_s e^{-\beta\delta_s}}{\left(1+e^{-\beta\delta_s}\right)^2}\left(\frac{w_{i-1}}{s_i-s_{i-1}}+\frac{w_{i+1}}{s_{i+1}-s_i}\right)
	\label{eqn:sum}
	\end{equation} 
	If $\delta_s = 0$ in \eqref{eqn:sum} then $D = 0$, and otherwise the quotient
  $\tfrac{\beta e^{\beta\delta_s}}{\left(1+e^{\beta\delta_s}\right)^2}$ can be
  made arbitrarily small by choosing large $\beta$. We also note that
  $\tfrac{\beta e^{\beta\delta_s}}{\left(1+e^{\beta\delta_s}\right)^2} =
  \tfrac{\beta e^{-\beta\delta_s}}{\left(1+e^{-\beta\delta_s}\right)^2}$ meaning
  the same result follows for $\delta_s > 0$ and $\delta_s < 0$. The statement
  of the lemma follows by applying the same summation for all values of $i$.
\end{proof}

\subsection{Energy sufficiency }
%After representing the path between the robot and the charging station in the way laid out in previous section , we herein describe the CBF based framework we adopt to enable the robot achieving energy sufficiency while moving along the smooth piecewise linear path.

We consider the case in which the robot lies at the beginning of the smooth path
\eqref{eqn:path_point} and moves along this path back to the station (at the
other end of the path). We assume the path is static, i.e. not changing.

We define a reference point along the path as in \eqref{eqn:path_point} 
\begin{equation}
\begin{split}
x_r(s) &= p(s)= \sum_{i=1}^{n_w-1}\sigma_{i}(s)\left(\tfrac{s_{i+1}-s}{s_{i+1}-s_i} w_i + \tfrac{s-s_i}{s_{i+1}-s_i}w_{i+1}\right)\\
\frac{\partial x_r}{\partial s} &= \sum_{i=1}^{n_w-1}\sigma_{i}(s)\frac{w_{i+1}-w_i}{s_{i+1}-s_i}
\end{split}
\label{eqn:ref_pt}
\end{equation} 
noting that the derivative expression follows from Lemma~\ref{lemma:diff_sum}.
We want to control the value of $s$ in a way that makes the reference point
approach the end of path in a manner commensurate to the robot's energy content.
For this purpose, we introduce the following dynamics for $s$
\begin{equation}
\dot{s} = \eta
\end{equation}
with $\eta\in\mathbb{R}$ and $s(0) = 0$. The outline of our strategy is as
follows: we introduce constraints that manipulate the value of $s$ in a way that
makes the reference point $x_r$ approach the end of path as the total energy
content decreases, and use an additional constraint to make the robot follow
$x_r$. The candidate CBF for energy sufficiency is
\begin{equation}
h_e = E_{nom} - E - \tfrac{\mathcal{P}(v_r)}{v_r}(L(1-s) - \delta)
\label{eqn:energy_suff_cbf}
\end{equation}
where $v_r$ is the desired velocity with which the robot moves along the path,
$\delta$ is the distance of the boundary of charging region away from its
center, noting that the center of the charging region is $w_{n_w}$. We note is
the expression $L(1-s)$ expresses the length along the path from point $x_r(s)$ till its end. The
constraint $\dot{h}_e\geq -\alpha(h_e)$ associated with this candidate CBF is
\begin{equation}
\begin{split}
&-\mathcal{P}(u)+\tfrac{\mathcal{P}(v_r)}{v_r}L\eta \geq -\gamma_e h_e
\end{split}
\label{eqn:suff_constraint}
\end{equation}

In \eqref{eqn:energy_suff_cbf} the value of $s$ needs to be maintained above
zero (otherwise the value of $h_e$ can be still positive without having the
reference point $x_r$ moving back towards the end of the path). For this end we
introduce a constraint that lower bounds $s$ with the following candidate CBF
\begin{equation}
h_b = s
\label{eqn:lower_bound}
\end{equation}
with the associated constraint
\begin{equation}
\eta \geq -\gamma_bh_b
\label{eqn:lower_bound_constraint}
\end{equation}
We complement \eqref{eqn:energy_suff_cbf} and \eqref{eqn:lower_bound} with
another candidate CBF that aims at making the robot follow $x_r(s)$ as it
changes, and is defined as follows
\begin{equation}
h_d = \tfrac{1}{2}(d^2 - ||x-x_r(s)||^2)
\label{eqn:tracking_cbf}
\end{equation} 
with $0<d< \delta$. The constraint associated with this candidate CBF is 
\begin{equation}
-(x-x_r(s))^T(u-\dot{x}_r(s))\geq -\gamma_d h_d
\label{eqn:tracking_constraint}
\end{equation}
where $\dot{x}_r=\tfrac{\partial x_r}{\partial s}\eta$.

In the following lemmas we show that the proposed CBFs lead the robot back to
the charging station with $E_{nom} - E \geq 0$. We note that we are not
controlling $u$ in \eqref{eqn:suff_constraint} but rather give this task to
\eqref{eqn:tracking_constraint}, thus partially decoupling the reference point's
movement from the robot's control action. \emph{In other words, we deliberately
  make the system respond to changing energy levels by moving the reference
  point along the path without directly changing the robot's velocity.} This
interplay between energy sufficiency and tracking constraints is highlighted in
the next lemma.

\begin{lemma}
	\label{lemma:base}
	For a robot with dynamics described in \eqref{eqn:robot_dynamics} and power
  consumption as in \eqref{eqn:power_parabolic}, and has a maximum
  magnitude of control action
  $u_{\text{max}}$, the control barrier functions defined in
  \eqref{eqn:energy_suff_cbf} and \eqref{eqn:tracking_cbf} are zeroing control
  barrier functions (ZCBF) provided that
%	\begin{equation*}
%	\frac{\mathcal{P}(u)}{\mathcal{P}(v_r)} v_r\leq u_{\text{max}} 
%	\end{equation*}
	\begin{equation*}
	v_{r}^{*} = \sqrt{\frac{m_0}{m_2}}\leq u_{\text{max}} 
	\end{equation*}
	where $||u||\leq u_{max}$. Moreover, provided that $L(s)>\delta$, then $E=E_{nom}$ only at  $L(1-s) = \delta$.
%	, i.e. on the boundary of the charging region. 
%	Moreover, satisfying \eqref{eqn:suff_constraint} leads $x_r$ to move along the path with speed equal to $v_r$.
\end{lemma}
\begin{proof}
	The idea of the proof is to show that there is always a value of $\eta$ that
  satisfies \eqref{eqn:suff_constraint} with its $\mathcal{P}(u)$ term, and
  there is always $u$ to satisfy \eqref{eqn:tracking_constraint} at the same
  time. Since $\eta\in\mathbb{R}$ means there is always a value of $\eta$
  that satisfies \eqref{eqn:suff_constraint}, thus \eqref{eqn:energy_suff_cbf}
  is a ZCBF. However, the reference point $x_r$ could be moving with a speed too
  fast for the robot to track depending on the value of $\mathcal{P}(u)$.
	
	We consider the critical case of approaching the boundary of the safe set for
  both $h_d$ and $h_e$, i.e. $h_e\approx 0$ and $h_d\approx 0$, in which case we
  can consider the equality condition of the constraints
  \eqref{eqn:suff_constraint} and \eqref{eqn:tracking_constraint} (i.e. near the
  boundary of the safe set the safe actions should at least satisfy $\dot{h} =
  -\alpha h$ for both \eqref{eqn:suff_constraint} and
  \eqref{eqn:tracking_constraint}). The aforementioned constraints become
%	 Substituting \eqref{eqn:len_bez_diff} in \eqref{eqn:suff_constraint} and tacking the equality
	\begin{subequations}
		\begin{equation}
		\eta = \frac{\mathcal{P}(u)}{\mathcal{P}(v_r)}\frac{v_r}{L}
		\label{eqn:eta_constraint_relation}
		\end{equation}
		\begin{equation}
		u = \frac{\partial x_r}{\partial s}\eta = \frac{\partial x_r}{\partial s}\frac{\mathcal{P}(u)}{\mathcal{P}(v_r)}\frac{v_r}{L}
		\label{eqn:speed_constraint_relation}
		\end{equation}
		\label{eqn:const_eq}
	\end{subequations}
	noting that $x-x_r \neq 0$ when $h_d\approx 0$. We also note that
	\begin{equation*}
	||u|| = \frac{\mathcal{P}(u)}{\mathcal{P}(v_r)}\frac{v_r}{L}	\left\lVert\frac{\partial x_r}{\partial s}\right\rVert.
	\end{equation*}
	Assuming $\beta\gg 1$, the derivative $\tfrac{\partial x_r}{\partial s}$ is as described in \eqref{eqn:path_point}. Moreover, 
	\begin{equation}
	s_{i+1}-s_i = \frac{\sum_{k=1}^{i}\ell_k - \sum_{k=1}^{i-1}\ell_k}{L} = \frac{||w_{i+1}-w_i||}{L}
	\end{equation} 
	where $\ell_k = ||w_{k+1}-w_k||$ and consequently $\tfrac{\partial x_r}{\partial s}$ can be expressed as
	\begin{equation}
		\tfrac{\partial x_r}{\partial s} = L\sum_{i = 1}^{n_w-1}\sigma_i(s)\hat{e}_i
		\label{eqn:simple_sig_diff}
	\end{equation}
	where $\hat{e}_i=\tfrac{w_{i+1}-w_i}{||w_{i+1}-w_i||}$ is a unit vector. To
  estimate $\left\lVert\tfrac{\partial x_r}{\partial s}\right\rVert$ it suffices
  to mention that in the range $s_{i}+\epsilon_m < s < s_{i+1}-\epsilon_m$ (for
  $i=1,\dots,n_w-1$ and $\epsilon_m = \tfrac{2}{\beta}$) all the double sigmoid
  functions in \eqref{eqn:simple_sig_diff} will be almost equal to zero except
  for one (by definition) and thus $\left\lVert\tfrac{\partial x_r}{\partial
      s}\right\rVert =L$. Moreover, if $s_i-\epsilon_m<s<s_i+\epsilon_m$, i.e.
  $s$ is transitioning from one segment to the next, the sum of the two sigmoid
  functions locally around $s=s_i$ is equal to one as show in
  \eqref{eqn:trans_sigs}, meaning that \eqref{eqn:simple_sig_diff} will be a
  convex sum of two unit vectors which will have at most a magnitude equal to
  one so $\left\lVert\tfrac{\partial x_r}{\partial s}\right\rVert \leq L$.
  Therefore $||u||$ becomes
	\begin{equation}
	||u|| = \frac{\mathcal{P}(u)}{\mathcal{P}(v_r)}v_r
	\label{eqn:abs_speed}
	\end{equation}
	which is a root finding problem for a polynomial of the second degree, since
  $\mathcal{P}(u)$ is a second order polynomial \eqref{eqn:power_parabolic}.
  Solving for the roots we get
	\begin{equation}
	\lambda_1 = v_r,\quad \lambda_2 = \frac{m_0}{m_2v_r}
	\label{eqn:roots}
	\end{equation}
	and these roots are equal when $v_{r}^{*} = \sqrt{\tfrac{m_0}{m_2}}$. Since we consider a case where the robot is stationary and the path is fixed,
    the robot starts from this stationary state and converges to $||u||=
  \min(\lambda_1,\lambda_2)$ for a given value of $v_r$. This means that the
  maximum achievable return velocity is at $v_r = v_{r}^{*}$ where $\lambda_1 =
  \lambda_2$. If $v_{r}^{*} \leq u_{max}$ then there is always a control action
  $u$ available to satisfy \eqref{eqn:tracking_constraint}, rendering
  \eqref{eqn:tracking_cbf} a ZCBF. If $h_e=0$ then from
  \eqref{eqn:energy_suff_cbf} $E = E_{nom}$ can only happen if $L(1-s) =
  \delta$, meaning the remaining length along the path is equal to delta, which
  only happens at the boundary of the charging region.
%	Since \eqref{eqn:energy_suff_cbf} is a ZCBF and \eqref{eqn:suff_constraint} is one of the constraints of the Quadratic program \eqref{eqn:QP} then $\dot{h}_e = -\alpha_eh_e$ applies
\end{proof} 

\begin{remark}
	\label{remark:stability}
	The previous proof assumes the presence of \emph{a-priori} known model for
  power consumption. However, a mismatch between the power model
  $\mathcal{P}(u)$ in \eqref{eqn:power_parabolic} and the actual power
  consumption $\bar{\mathcal{P}}(u)$ will lead to a different solution of \eqref{eqn:abs_speed}. We are
  interested in the case where $\bar{\mathcal{P}}(u) =
  \mathcal{P}(u)+\Delta_p$, with
  $\Delta_p\in\mathbb{R}$. The root finding problem in \eqref{eqn:abs_speed}
  becomes
	\begin{equation}
	||u|| = \frac{\bar{\mathcal{P}}(u)}{\mathcal{P}(v_r)}v_r
	\label{eqn:mod_roots}
	\end{equation}
	and the roots will be
	\begin{equation}
	\bar{\lambda}_{1,2} = \frac{m_0+m_2v_{r}^{2}\pm\mathcal{D}}{2m_2v_r}
	\end{equation}
	where $\mathcal{D}=\sqrt{(m_0-m_2v_{r}^{2})^2-4m_2v_{r}^{2}\Delta_p}$. When
  $\Delta_p=0$, $\bar{\lambda}_{1,2}=\lambda_{1,2}$ as described in
  \eqref{eqn:roots}. If $\Delta_p > 0$, then $\mathcal{D} < (m_0-m_2v_{r}^{2})$
  and as a result $\bar{\lambda}_1 > v_r$ and $\bar{\lambda}_2 <
  \tfrac{m_0}{m_2v_r}$. In other words, the robot will converge to a faster
  speed in case the actual power consumption is more than expected, and the
  converse is true for $\Delta_p < 0$. When
	\begin{equation}
	\Delta_p > \Delta_{p}^{*}= \left(\tfrac{m_0-m_2v_{r}^{2}}{2v_r\sqrt{m_2}}\right)^2
	\label{eqn:stability_margin}
	\end{equation}
	$\mathcal{D}$ becomes undefined and there will be no roots for
  \eqref{eqn:mod_roots}, indicating a point of instability in velocity for power
  disturbances beyond $\Delta_{p}^{*}$. This idea is illustrated in
  Figure~\ref{fig:roots}.
\end{remark}

\begin{figure}[!htb]
	\centering
	\includegraphics[width=\columnwidth]{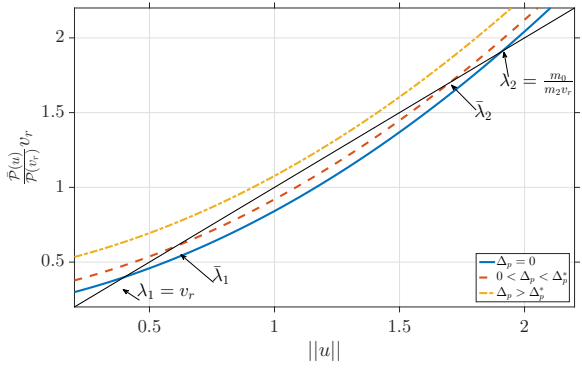}
	\caption{Graphical representation for the roots of \eqref{eqn:mod_roots} for different values of disturbance power $\Delta_p$. The roots are intersections of the straight line $f_1(u)=||u||$ in black and the parabolas $f_2(u)=\tfrac{\bar{\mathcal{P}}(u)}{\mathcal{P}(v_r)}v_r$ (representing RHS and LHS of \eqref{eqn:mod_roots} respectively). }
	\label{fig:roots}
\end{figure}

\begin{lemma}
	\label{lemma:2}
	For a robot with dynamics \eqref{eqn:robot_dynamics}, the candidate CBF
  \eqref{eqn:lower_bound} is a ZCBF.
\end{lemma}
\begin{proof}
	Since $\eta\in\mathbb{R}$ then there exist a value of $\eta$ that satisfies
  \eqref{eqn:lower_bound_constraint}. We need to show that this constraint does
  not conflict with \eqref{eqn:suff_constraint} when both constraints are on the
  boundary of their respective safe sets, i.e. $h_e = h_b = 0$. From
  \eqref{eqn:suff_constraint}
	\begin{equation}
	\eta \geq \frac{\mathcal{P}(u)}{\bar{\mathcal{P}}(v_r)}\frac{v_r}{L}
	\label{eqn:suff_const_simple}
	\end{equation}
	while \eqref{eqn:lower_bound_constraint} becomes $\eta \geq 0$. Since the
  right hand side of \eqref{eqn:suff_const_simple} is always positive, it means
  there is always a value of $\eta$ that satisfies both
  \eqref{eqn:suff_const_simple} and \eqref{eqn:lower_bound_constraint}, thus
  \eqref{eqn:lower_bound} is a ZCBF.
\end{proof}

Although from Lemma \ref{lemma:base} we show that $E(t)=E_{nom}$ on the boundary
of the charging region, this is a result that concerns the reference point's
position $x_r$ while the robot's actual position tracks $x_r$ through enforcing
the constraint \eqref{eqn:tracking_cbf}. This situation implies the possibility
of $x_r$ reaching a point where $L(s) = \delta$ (boundary of charging region)
while the robot's position is lagging behind. In other words, we need the instant
where $E(t) = E_{nom}$ to happen inside of the charging region or at least on
its boundary.

\begin{proposition}
	\label{proposition:length_error}
	Consider a robot with dynamics \eqref{eqn:robot_dynamics} and applying the
  constraints pertaining to the CBFs \eqref{eqn:energy_suff_cbf},
  \eqref{eqn:lower_bound} and \eqref{eqn:tracking_cbf}. We define a modified
  distance threshold $\delta_m$ as
	\begin{equation}
	\delta_m \leq \delta - d
	\label{eqn:delta_m}
	\end{equation}
	then using $\delta_m$ in \eqref{eqn:energy_suff_cbf} ensures that $E(t)$ will be at most equal to $E_{nom}$.
\end{proposition}
\begin{proof}
	\begin{figure}[!htb]
		\centering
		\fontsize{30pt}{30pt}\selectfont% or whatever fontsize you like
		\def\svgwidth{5.333in}
		\scalebox{0.37}{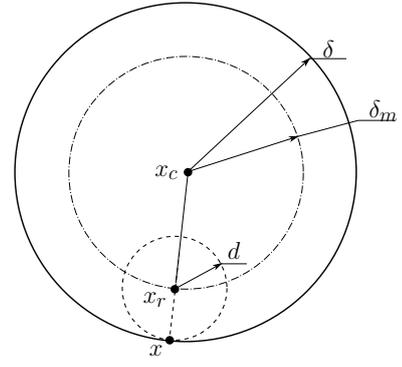}
		\caption{Demonstration of $x_r$ pursuing $\delta_m$ as the boundary of the charging region in \eqref{eqn:energy_suff_cbf} while having a robot following the reference point $x_r$ at a distance $d$ away. Here $x_r$ is the reference point position, $x_c$ is the charging station center position, $\delta$ is the charging region's radius, and $\delta_m$ is a reduced radius to track as described in \eqref{eqn:delta_m}.}
		\label{fig:charging_region}
	\end{figure}
	From Lemma~\ref{lemma:base} $E = E_{nom}$ only at $L(1-s)=\delta_m$, which is
  equal to the length of $(x_c,x_r)$ segment in
  Figure~\ref{fig:charging_region}, i.e. the remaining length along the path
  from $x_r$ to $x_c$. This implies that $||x-x_c||\leq \delta$ as demonstrated
  in Figure~\ref{fig:charging_region}.
	
\end{proof}

\begin{theorem}
	\label{thm:static_suff}
	For a robot with dynamics described by \eqref{eqn:robot_dynamics} and maximum
  magnitude of control action $u_{max}$, applying the QP in
  \eqref{eqn:QP} with constraints \eqref{eqn:suff_constraint} (with $\delta =
  \delta_m$ from \eqref{eqn:delta_m}), \eqref{eqn:tracking_constraint} and
  \eqref{eqn:lower_bound_constraint}, and with a static piecewise linear path
  with waypoints $\mathcal{W}$, then energy sufficiency is maintained , i.e. $E
  < E_{nom}$ if $||x-x_c|| > \delta$.
\end{theorem}
\begin{proof}
	If we substitute $\delta_m$ from \eqref{eqn:delta_m} in
  \eqref{eqn:energy_suff_cbf} and \eqref{eqn:suff_constraint}, then from lemma
  \ref{lemma:base} $E = E_{nom}$ iff $L(1-s) = \delta_m$, and from proposition
  \ref{proposition:length_error} this implies that $||x-x_c||<\delta$ when $E =
  E_{nom}$, and since $E$ is strictly increasing (due to $\mathcal{P}(u)>0$ by
  definition), we conclude that $E \leq E_{nom}$ in $||x-x_c||>\delta$.
\end{proof}

%%% Local Variables:
%%% mode: latex
%%% TeX-master: "ijrr_draft"
%%% End:

%% file: pw_path.pdf_tex
%% Creator: Inkscape inkscape 0.92.5, www.inkscape.org
%% Accompanies image file 'pw_path.pdf' (pdf, eps, ps)
\begingroup%
  \makeatletter%
  \providecommand\color[2][]{%
    \errmessage{(Inkscape) Color is used for the text in Inkscape, but the package 'color.sty' is not loaded}%
    \renewcommand\color[2][]{}%
  }%
  \providecommand\transparent[1]{%
    \errmessage{(Inkscape) Transparency is used (non-zero) for the text in Inkscape, but the package 'transparent.sty' is not loaded}%
    \renewcommand\transparent[1]{}%
  }%
  \providecommand\rotatebox[2]{#2}%
  \newcommand*\fsize{\dimexpr\f@size pt\relax}%
  \newcommand*\lineheight[1]{\fontsize{\fsize}{#1\fsize}\selectfont}%
  \ifx\svgwidth\undefined%
    \setlength{\unitlength}{376.52399979bp}%
    \ifx\svgscale\undefined%
      \relax%
    \else%
      \setlength{\unitlength}{\unitlength * \real{\svgscale}}%
    \fi%
  \else%
    \setlength{\unitlength}{\svgwidth}%
  \fi%
  \global\let\svgwidth\undefined%
  \global\let\svgscale\undefined%
  \makeatother%
  \begin{picture}(1,0.5980262)%
    \lineheight{1}%
    \setlength\tabcolsep{0pt}%
    \put(0,0){\includegraphics[width=\unitlength,page=1]{pw_path.pdf}}%
    \put(0.14669794,0.04572966){\color[rgb]{0,0,0}\makebox(0,0)[lt]{\lineheight{1.25}\smash{\begin{tabular}[t]{l}$w_1$\end{tabular}}}}%
    \put(0.0469165,0.30780462){\color[rgb]{0,0,0}\makebox(0,0)[lt]{\lineheight{1.25}\smash{\begin{tabular}[t]{l}$w_2$\end{tabular}}}}%
    \put(0.4110082,0.53104017){\color[rgb]{0,0,0}\makebox(0,0)[lt]{\lineheight{1.25}\smash{\begin{tabular}[t]{l}$w_3$\end{tabular}}}}%
    \put(0.6234306,0.42803033){\color[rgb]{0,0,0}\makebox(0,0)[lt]{\lineheight{1.25}\smash{\begin{tabular}[t]{l}$w_4$\end{tabular}}}}%
    \put(0,0){\includegraphics[width=\unitlength,page=2]{pw_path.pdf}}%
    \put(0.83917394,0.5440646){\color[rgb]{0,0,0}\makebox(0,0)[lt]{\lineheight{1.25}\smash{\begin{tabular}[t]{l}$w_5$\end{tabular}}}}%
    \put(0.23864897,0.04498102){\color[rgb]{0,0,0}\makebox(0,0)[lt]{\lineheight{1.25}\smash{\begin{tabular}[t]{l}$s_1=0$\end{tabular}}}}%
    \put(0.14191939,0.30791237){\color[rgb]{0,0,0}\makebox(0,0)[lt]{\lineheight{1.25}\smash{\begin{tabular}[t]{l}$s_2=\frac{L_1}{L}$\end{tabular}}}}%
    \put(0.30553113,0.39856485){\color[rgb]{0,0,0}\makebox(0,0)[lt]{\lineheight{1.25}\smash{\begin{tabular}[t]{l}$s_3=\frac{L_1+L_2}{L}$\end{tabular}}}}%
    \put(0.59854169,0.34298411){\color[rgb]{0,0,0}\makebox(0,0)[lt]{\lineheight{1.25}\smash{\begin{tabular}[t]{l}$s_4=\frac{L_1+L_2+L_3}{L}$\end{tabular}}}}%
    \put(0.85473229,0.46548446){\color[rgb]{0,0,0}\makebox(0,0)[lt]{\lineheight{1.25}\smash{\begin{tabular}[t]{l}$s_5=1$\end{tabular}}}}%
    \put(0,0){\includegraphics[width=\unitlength,page=3]{pw_path.pdf}}%
    \put(0.02186705,0.12121884){\color[rgb]{0,0,0}\makebox(0,0)[lt]{\lineheight{1.25}\smash{\begin{tabular}[t]{l}$L_1$\end{tabular}}}}%
    \put(0.17182459,0.47179398){\color[rgb]{0,0,0}\makebox(0,0)[lt]{\lineheight{1.25}\smash{\begin{tabular}[t]{l}$L_2$\end{tabular}}}}%
    \put(0.57504277,0.51248564){\color[rgb]{0,0,0}\makebox(0,0)[lt]{\lineheight{1.25}\smash{\begin{tabular}[t]{l}$L_3$\end{tabular}}}}%
    \put(0.85447864,0.27829456){\color[rgb]{0,0,0}\makebox(0,0)[lt]{\lineheight{1.25}\smash{\begin{tabular}[t]{l}$L_4$\end{tabular}}}}%
    \put(0,0){\includegraphics[width=\unitlength,page=4]{pw_path.pdf}}%
    \put(0.19131678,0.42156952){\color[rgb]{0,0,0}\makebox(0,0)[lt]{\lineheight{1.25}\smash{\begin{tabular}[t]{l}$\textcolor{orange}{p(s)}$\end{tabular}}}}%
    \put(0,0){\includegraphics[width=\unitlength,page=5]{pw_path.pdf}}%
  \end{picture}%
\endgroup%

%% file: charging_region3.pdf_tex
%% Creator: Inkscape inkscape 0.92.5, www.inkscape.org
%% Accompanies image file 'charging_region3.pdf' (pdf, eps, ps)
\begingroup%
  \makeatletter%
  \providecommand\color[2][]{%
    \errmessage{(Inkscape) Color is used for the text in Inkscape, but the package 'color.sty' is not loaded}%
    \renewcommand\color[2][]{}%
  }%
  \providecommand\transparent[1]{%
    \errmessage{(Inkscape) Transparency is used (non-zero) for the text in Inkscape, but the package 'transparent.sty' is not loaded}%
    \renewcommand\transparent[1]{}%
  }%
  \providecommand\rotatebox[2]{#2}%
  \newcommand*\fsize{\dimexpr\f@size pt\relax}%
  \newcommand*\lineheight[1]{\fontsize{\fsize}{#1\fsize}\selectfont}%
  \ifx\svgwidth\undefined%
    \setlength{\unitlength}{370.66119337bp}%
    \ifx\svgscale\undefined%
      \relax%
    \else%
      \setlength{\unitlength}{\unitlength * \real{\svgscale}}%
    \fi%
  \else%
    \setlength{\unitlength}{\svgwidth}%
  \fi%
  \global\let\svgwidth\undefined%
  \global\let\svgscale\undefined%
  \makeatother%
  \begin{picture}(1,0.93943758)%
    \lineheight{1}%
    \setlength\tabcolsep{0pt}%
    \put(0,0){\includegraphics[width=\unitlength,page=1]{charging_region3.pdf}}%
    \put(0.56073245,0.25872703){\color[rgb]{0,0,0}\makebox(0,0)[lt]{\lineheight{1.25}\smash{\begin{tabular}[t]{l}$d$\end{tabular}}}}%
    \put(0,0){\includegraphics[width=\unitlength,page=2]{charging_region3.pdf}}%
    \put(0.81046245,0.79322316){\color[rgb]{0,0,0}\makebox(0,0)[lt]{\lineheight{1.25}\smash{\begin{tabular}[t]{l}$\delta$\end{tabular}}}}%
    \put(0.36969604,0.47982148){\color[rgb]{0,0,0}\makebox(0,0)[lt]{\lineheight{1.25}\smash{\begin{tabular}[t]{l}$x_c$\end{tabular}}}}%
    \put(0,0){\includegraphics[width=\unitlength,page=3]{charging_region3.pdf}}%
    \put(0.33847161,0.14272509){\color[rgb]{0,0,0}\makebox(0,0)[lt]{\lineheight{1.25}\smash{\begin{tabular}[t]{l}$x_r$\end{tabular}}}}%
    \put(0.35478226,0.00475819){\color[rgb]{0,0,0}\makebox(0,0)[lt]{\lineheight{1.25}\smash{\begin{tabular}[t]{l}$x$\end{tabular}}}}%
    \put(0,0){\includegraphics[width=\unitlength,page=4]{charging_region3.pdf}}%
    \put(0.93139067,0.63494809){\color[rgb]{0,0,0}\makebox(0,0)[lt]{\lineheight{1.25}\smash{\begin{tabular}[t]{l}$\delta_m$\end{tabular}}}}%
    \put(0,0){\includegraphics[width=\unitlength,page=5]{charging_region3.pdf}}%
  \end{picture}%
\endgroup%

%% file: dynamic_energy.tex
We extend the results from the previous section to consider the case in which
the path is changing with time due to robot's movement and replanning actions.
\subsection{Effect of robot's movement}
Assuming that the path is fixed (i.e., there is no replanning), the main
difference from the static case is that the first waypoint $w_1\in\mathcal{W}$
is the robot's position, leading to a change in the total path length $L$ as the
robot moves. Additionally, the values of $s_i$ at the different waypoints will
change as a result. We therefore consider the following simple proportional
control dynamics for $w_1$:
\begin{equation}
\dot{w}_1 =\xi =-k_w(w_1-x)
\end{equation}
where $K_w \gg 0$ and $\xi\in\mathbb{R}$.
The change in total path length is:
\begin{equation}
\begin{split}
\dot{L} &= \frac{d}{dt}\left(\sum_{i=2}^{n_w-1}||w_{i+1}-w_i|| + ||w_2-w_i||\right)\\
&= -\frac{(w_2-w_1)}{||w_{2}-w_1||}\xi,
\end{split}
\end{equation}
noting that all the waypoints other than $w_1$ are fixed. The change in $s_i$ is
\begin{equation}
\dot{s}_i =\frac{d}{dt}\frac{L_i}{L}= \frac{d}{dt}\left(1-\frac{\bar{L}_i}{L}\right)=\frac{\bar{L}_i}{L^2}\dot{L}
\end{equation}
where $\bar{L}_i$ is the length along the path from waypoint $w_i$ to the end of
the path and is constant for $i=2,\dots,n_w-1$. The derivative
$\tfrac{dx_r}{dt}$ is:
\begin{equation}
\frac{dx_r}{dt} = \frac{\partial x_r}{\partial s}\eta + \frac{\partial x_r}{\partial t}
\label{eqn:xr_dt}
\end{equation}
where $\tfrac{\partial x_r}{\partial t}$ follows from differentiating
\eqref{eqn:path_point} with respect to time:
\begin{equation}
\frac{\partial x_r}{\partial t} =  \sum_{i=1}^{n_w-1}\sigma_i(s)\left(w_{i+1}-w_i\right)\tfrac{\dot{L}}{L^2}\tfrac{\bar{L}_i(s_{i+1}-s)+\bar{L}_{i+1}(s-s_i)}{(s_{i+1}-s_i)^2}.
\end{equation}
Consequently, the energy sufficiency constraint \eqref{eqn:suff_constraint}
becomes
\begin{equation}
\begin{split}
&-\mathcal{P}(u)+\tfrac{\bar{\mathcal{P}}(v_r)}{v_r}(L\eta - \dot{L}(1-s))\geq -\gamma_e h_e
\end{split}
\label{eqn:suff_constraint_dynamic}
\end{equation}
and the tracking constraint \eqref{eqn:tracking_constraint} now uses $\dot{x}_r$ as in \eqref{eqn:xr_dt}.

The results from Theorem~\ref{thm:static_suff} rely on the fact that the path is
static. To use the same result in the dynamic case we
``freeze'' the path when the robot needs to go back to recharge, i.e. we
stop $w_1$ from tracking robot's position when it needs to go back to
recharge:

\input{proposition_2_energy.tex}

The full QP problem with the constraints discussed so far can be expressed as
\begin{equation}
\begin{aligned}
&  \mathbf{u}^*=\underset{\mathbf{u}}{\text{min}}
& & ||\mathbf{u}-\mathbf{u}_{nom}||^2 \\
& \quad \quad \quad \text{s.t.}
& & \mathbf{A}\mathbf{u} \geq \mathbf{B} \\
%&&& X \succeq 0.
\end{aligned}
\label{eqn:QP_suff}
\end{equation}
where 
\begin{equation}
\begin{split}
\mathbf{A} &= \begin{bmatrix}
\tfrac{\mathcal{P}(v_r)}{v_r}L &\mathbf{0}_{1\times 2}\\1&\mathbf{0}_{1\times 2}\\
(x-x_r)^T\tfrac{\partial x_r}{\partial s}&-(x-x_r)^T
\end{bmatrix}\\ 
\mathbf{B} &= \begin{bmatrix}
-\gamma_e h_e + \mathcal{P}(u)+\dot{L}(1-s)\\-\gamma_bh_b\\
-\gamma_dh_d
\end{bmatrix}\\
\mathbf{u}_{nom}&=\begin{bmatrix}
0&u_{nom}
\end{bmatrix}
\end{split}
\label{eqn:defs}
\end{equation}

%\Rightarrow&-k_eV-\frac{k_e}{v_r}(V-h_e)\frac{dL}{dt} \geq -\gamma_e h_e

\begin{theorem}
	\label{thm:suff}
	For a robot described by \eqref{eqn:robot_dynamics} with a set of ordered
  waypoints $\mathcal{W}\in\mathbb{R}^{n_w\times 2}$, \eqref{eqn:QP_suff}
  ensures energy sufficiency.
\end{theorem}
\begin{proof}
	Since \eqref{eqn:energy_suff_cbf} and \eqref{eqn:tracking_cbf} are valid ZCBFs
  from Proposition \ref{proposition:freeze} and \eqref{eqn:lower_bound} is a
  valid ZCBF then from Lemma~\ref{lemma:2}, then from
  Proposition~\ref{proposition:length_error} and Lemma~\ref{lemma:base}
  (augmented by Proposition~\ref{proposition:freeze}) $E(t) = E_{nom}$ at
  $||x-x_c|| < \delta$ (inside the charging region), and since $E(t)$ is
  strictly increasing it means $E(t) < E_{nom}$ for $||x-x_c|| > \delta$,
  as shown by Theorem~\ref{thm:static_suff}, i.e. energy sufficiency is
  maintained.
\end{proof}
\subsection{Effect of path planning}
During the course of a mission, the path planner keeps updating the waypoints
back to the charging station every $\mathcal{T}$ seconds, meaning that there are
discrete changes in the number of waypoints and their locations, which can lead
to violating the energy sufficiency constraint. To account for these changes, we
impose some conditions on the output of the path planner so as not to violate
other constraints.

\begin{definition}
	\label{def:spc}
	Assuming there is a path $\mathcal{W}^{(k-1)\mathcal{T}} =
  \{w_1^{(k-1)\mathcal{T}},\dots,w_{n_w}^{(k-1)\mathcal{T}}\}$ at time
  $(k-1)\mathcal{T}$ between a robot at position
  $x(k\mathcal{T})=w_1^{(k-1)\mathcal{T}}$ and the charging station, Sequential
  Path Construction (SPC) is the process of creating a new set of waypoints
  $\mathcal{W}^{k\mathcal{T}}=\{w_1^{k\mathcal{T}},\dots,w_{n_w+1}^{k\mathcal{T}}\}$
  at time $k\mathcal{T}$ provided that $\zeta(s) = 0$, where $\zeta(s)$ is
  defined in \eqref{eqn:switch}, such that
	\begin{equation}
	\begin{split}
	w_1^{k\mathcal{T}} &= x(k\mathcal{T})\\
	w_2^{k\mathcal{T}}&=\kappa w_1^{k\mathcal{T}}+(1-\kappa)w_3^{k\mathcal{T}}\\
	w_{i+1}^{k\mathcal{T}}&=w_i^{(k-1)\mathcal{T}}, \quad i=2,\dots,n_w
	\end{split}
	\label{eqn:spc}
	\end{equation}
	where $0\ll\kappa<1$.
\end{definition}
\begin{lemma}
	\label{lemma:length_invariance}
	Sequential Path Construction is path length and path angle invariant, meaning
  the following two equations are satisfied
	\begin{equation}
	\begin{split}
	\sum_{i=1}^{n_w-1}||w_i^{(k-1)\mathcal{T}}-w_{i+1}^{(k-1)\mathcal{T}}|| &= \sum_{i=1}^{n_w}||w_i^{k\mathcal{T}}-w_{i+1}^{k\mathcal{T}}||\\
	\sum_{i=2}^{n_w-1}|\psi_i^{(k-1)\mathcal{T}}|&=\sum_{i=2}^{n_w}|\psi_i^{k\mathcal{T}}|
	\end{split}
	\end{equation}
	where 
	\begin{equation}
	\psi_i = \cos^{-1}\frac{\Delta w_{i-1}^{i}.\Delta w_{i}^{i+1}}{||\Delta w_{i-1}^{i}||.||\Delta w_{i}^{i+1}||}
	\label{eqn:psi}
	\end{equation}
	and $\Delta w_{i}^{i+1} = w_{i+1}-w_i$
\end{lemma}
\begin{proof}
	The path length at time $k\mathcal{T}$ is 
	\begin{equation}
	\begin{split}
	&L^{k\mathcal{T}}=\sum_{i=1}^{n_w}||w_i^{k\mathcal{T}}-w_{i+1}^{k\mathcal{T}}||\\
	&=||w_1^{k\mathcal{T}}-w_2^{k\mathcal{T}}||+||w_2^{k\mathcal{T}}-w_3^{k\mathcal{T}}||+\sum_{i=3}^{n_w}||w_i^{k\mathcal{T}}-w_{i+1}^{k\mathcal{T}}||\\
	&=||w_1^{k\mathcal{T}}-w_3^{k\mathcal{T}}||+\sum_{i=2}^{n_w}||w_i^{(k-1)\mathcal{T}}-w_{i+1}^{(k-1)\mathcal{T}}||\\
	&=||w_1^{(k-1)\mathcal{T}}-w_2^{(k-1)\mathcal{T}}||+\sum_{i=2}^{n_w}||w_i^{(k-1)\mathcal{T}}-w_{i+1}^{(k-1)\mathcal{T}}||\\&=L^{(k-1)\mathcal{T}}.
	\end{split}
	\end{equation}
	%To show the path angle invariance,
  By definition we have:
  $\psi^{(k-1)\mathcal{T}}_i=\psi^{k\mathcal{T}}_{i+1}$ for $i=2,\dots,n_w$.
  Since
  $w_3^{k\mathcal{T}}-w_2^{k\mathcal{T}}=(1-\kappa)(w_3^{k\mathcal{T}}-w_1^{k\mathcal{T}})$
  and
  $w_2^{k\mathcal{T}}-w_1^{k\mathcal{T}}=\kappa(w_3^{k\mathcal{T}}-w_1^{k\mathcal{T}})$,
  then $\cos\psi_2^{k\mathcal{T}}=1$ therefore $\psi_2^{k\mathcal{T}}=0$. In
  conclusion
  $\sum_{i=2}^{n_w-1}|\psi_i^{(k-1)\mathcal{T}}|=\sum_{i=2}^{n_w}|\psi_i^{k\mathcal{T}}|$.
\end{proof}
\begin{proposition}
	\label{proposition:spc}
	Sequential Path Construction does not violate energy sufficiency, provided
  $\mathbf{u}_{nom}$ is Lipschitz.
\end{proposition}
\begin{proof}
	We need to show that changing the path with SPC does not affect the following
  two inequalities if they are satisfied for the original path:
	\begin{subequations}
		\begin{equation}
		h_e = E_{nom} - E - \tfrac{\mathcal{P}(v_r)}{v_r}(L(1-s)-\delta_m)\geq0
		\label{eqn:keep_suff_}
		\end{equation}
		\begin{equation}
		-\mathcal{P}(u_{nom}) + \tfrac{\mathcal{P}(v_r)}{v_r}\left(L\eta_{nom}-\dot{L}(1-s)\right)\geq -\gamma_e h_e
		\label{eqn:keep_d_suff_}
		\end{equation}
		\label{eqn:keep_}
	\end{subequations}
	Since $\eta_{nom}$ and $u_{nom}$ are continuous, then $\mathcal{P}(u_{nom}),
  E$ and $s$ are all continuous with no jumps (i.e., discrete changes). Since
  path length is invariant under SPC by virtue of
  Lemma~\ref{lemma:length_invariance}, then $L$ in \eqref{eqn:keep_} does not
  change as well as $\dot{L}$, meaning that \eqref{eqn:keep_} is not violated
  under SPC.
\end{proof}
Based on the proposition above, we introduce the Algorithm~\ref{algo} that
admits new paths produced by the path planner as long as they do not violate
\eqref{eqn:keep_}, and otherwise switches to SPC.
\begin{algorithm}[!h]
	\caption{Admission of new path at time $k\mathcal{T}$}
	\label{algo}
	\begin{algorithmic}
		\Require $\mathcal{W}^{(k-1)\mathcal{T}}$,$\mathcal{W}^{k\mathcal{T}}_{candidate}$, $x$,$\mathbf{u}_{nom}$
		\State $L,h_e\gets$ EVALUATE\_PATH($\mathcal{W}^{k\mathcal{T}}_{candidate},\mathbf{u}_{nom}$)
		\If{$\zeta(s) == 0$}
		\If{\eqref{eqn:keep_suff_} == False OR \eqref{eqn:keep_d_suff_} == False}
		\State $\mathcal{W}^{k\mathcal{T}} \gets$ SPC\_PATH$(x(k\mathcal{T}),\mathcal{W}^{(k-1)\mathcal{T}})$
		\Else {}
		\State $\mathcal{W}^{k\mathcal{T}} \gets \mathcal{W}^{k\mathcal{T}}_{candidate}$
		\EndIf
		\EndIf
		\State \Return $\mathcal{W}^{k\mathcal{T}}$
	\end{algorithmic}
\end{algorithm}
In Algorithm~\ref{algo},
EVALUATE\_PATH($\mathcal{W}^{k\mathcal{T}},\mathbf{u}_{nom}$) is a function that
takes the candidate path points from the path planner and evaluates the path
length, as well as the value of $h_e$, and
SPC\_PATH$(x(k\mathcal{T}),\mathcal{W}^{(k-1)\mathcal{T}})$ is a function that
updates a path using SPC.
\begin{theorem}
	\label{thm:SPC_suff}
	For a robot described by \eqref{eqn:robot_dynamics} that applies the control
  strategy in \eqref{eqn:QP_suff} for an already existing set of waypoints
  $\mathcal{W}^{(k-1)\mathcal{T}}$ at time $(k-1)\mathcal{T}$, $k\in
  \mathbb{N}$. Suppose a path planner produces a candidate set of waypoints
  $\mathcal{W}^{k\mathcal{T}}$ at time $k\mathcal{T}$ that satisfies the
  conditions in \eqref{eqn:keep_} and provided that $\zeta(s)=0$ from
  \eqref{eqn:switch}, then Algorithm~\ref{algo} ensures energy sufficiency is
  maintained.

%	\begin{subequations}
%		\begin{equation}
%		h_e = E_{nom} - E - \tfrac{\mathcal{P}(v_r)}{v_r}(L(1-s)-\delta_m)>0
%		\label{eqn:keep_suff}
%		\end{equation}
%		\begin{equation}
%		-\mathcal{P}(u_{nom}) + \tfrac{\mathcal{P}(v_r)}{v_r}\left(L\eta_{nom}-\dot{L}(1-s)\right)> -\gamma_e h_e
%		\label{eqn:keep_d_suff}
%		\end{equation}
%		\label{eqn:keep}
%	\end{subequations}
	
\end{theorem}
\begin{proof}
	If the new set of augmented waypoints satisfies \eqref{eqn:keep_suff_}, then
  switching from $\mathcal{W}^{(k-1)\mathcal{T}}$ to
  $\mathcal{W}^{k\mathcal{T}}$ does not violate the energy sufficiency
  constraint encoded by $h_e$. Moreover, if said switching satisfies
  \eqref{eqn:keep_d_suff_}, then $h_e > 0$ is satisfied with $\eta_{nom}$ at $s
  = 0$, meaning that $w_1$ tracks $x$ as outlined in
  Proposition~\ref{proposition:freeze} and consequently $x_r$ tracks the
  robot's position (since $s=0$ and $w_1$ tracks robot's position), thus $h_d >
  0$ is not violated as well, which means the sufficiency and tracking
  constraints are not violated by the path update.
	
	When $s> 0$, i.e. $\zeta(s)\neq0$, the path is frozen and energy sufficiency
  is maintained by virtue of Theorem~\ref{thm:suff}. If either condition in
  \eqref{eqn:keep_} is violated, the path is updated using SPC which maintains
  energy sufficiency as discussed in Proposition~\ref{proposition:spc}.
  Therefore Algorithm~\ref{algo} ensures energy sufficiency and tracking
  constraints are not violated.
%  \todo{The
%    proof lacks a conclusion. Also, I wouldn't refer to the algorithm in the
%    proof? Maybe the proof end is in the wrong spot}
	
%	the robot's current position can be added as a waypoint at the new sampling instance, which is a safe point already by virtue of theorem \eqref{thm:suff} and applying \eqref{eqn:QP_suff}.
\end{proof}
%\begin{remark}
%	The path update takes place when $\zeta(s)$ in \eqref{eqn:switch} is equal to zero, indicating that the reference point has not started moving along the path yet. As discussed in proposition~\ref{proposition:freeze} the main motivation is to carry out
%\end{remark}

%In algorithm \ref{algo} PATH\_LENGTH($\mathcal{W}^\mathcal{kT}$) is a procedure
%that length of the full path described by $\mathcal{W}^{k\mathcal{T}}$. We note
%that at moment of path update, the first waypoint on the path is the robot's
%position, i.e. $w_{1}^{k\mathcal{T}}=x(k\mathcal{T})$. The purpose
%of~\ref{algo} is essentially to make sure that the path planner does not
%produce a path that will be longer than
%PATH\_LENGTH($\mathcal{W}^{(k-1)\mathcal{T}}$) at the time of path update, and
%also to show that the discrete change in paths will not violate the maintenance
%of energy sufficiency.
%%% Local Variables:
%%% mode: latex
%%% TeX-master: "ijrr_draft"
%%% End:

%% file: proposition_2_energy.tex
\begin{proposition}
	\label{proposition:freeze}
	Consider a robot with dynamics \eqref{eqn:robot_dynamics} applying the
  proposed energy sufficiency framework described by the CBFs
  \eqref{eqn:energy_suff_cbf} and \eqref{eqn:lower_bound}. Consider the
  following dynamics for $w_1$
	\begin{equation}
	\dot{w}_1 =\xi= -k_{w}(w_1-x)\left(1-\zeta(s)\right)
	\label{eqn:p_con}
	\end{equation}
	where $\zeta$ is an activation function defined as
	\begin{equation}
	%	\xi(s)=\frac{1}{1+e^{K(s-\epsilon_{\zeta})}}
	\zeta(s) = \begin{cases}
	0\quad s \leq \epsilon_a\\
	1\quad\text{otherwise}
	\end{cases}
	\label{eqn:switch}
	\end{equation} 
	with   $0<\epsilon_a\ll\bar{\epsilon}_a<1$ and $||w_1-x_r(\bar{\epsilon}_a)||= d$, then \eqref{eqn:energy_suff_cbf} and \eqref{eqn:tracking_cbf} are ZCBF.
	%	Then energy sufficiency is maintained.
	%	Then \eqref{eqn:tracking_cbf} is a ZCBF. 
	%	\todo{we can add a condition on $\epsilon$ such that the reference point doesn't get out of the tracking circle around x in hd}
\end{proposition}
\begin{proof}
	We start by noting that \eqref{eqn:p_con} achieves tracking of the robot's
  position in case $s=0$, with an error inversely proportional to $k_w$,
  according to the candidate Lyapunov function
	\begin{equation}
	\scriptsize
	V = \tfrac{1}{2}(x-w_1)^T(x-w_1), \quad \dot{V} = (x-w_1)^Tu-k_w||x-w_1||^2
	\end{equation} 
	which has $\dot{V}\leq0$ under high value of $k_\omega$ (which is feasible since
  $\xi$ is an imaginary point with no physical characteristics). Moreover, since
  $\eta\in\mathbb{R}$ then there is a value of $\eta$ capable of satisfying the
  following inequality
	\begin{equation}
	\begin{split}
	&\eta \geq \frac{1}{L}\left(\tfrac{\mathcal{P}(u)-\gamma_eh_e}{\mathcal{P}(v_r)}v_r+\dot{L}(1-s)\right).
	\end{split}
	\label{eqn:suff_const_dyn}
	\end{equation}
  We need to show that if $ 0<s < \epsilon_a$ (when the reference point starts moving
  but the path freezing has not been activated yet, according to
  \eqref{eqn:switch}), tracking and energy sufficiency constraints are not
  violated as well as that $s$ increases so that $s > \epsilon_a$.
  
  To prove the latter, we need to show that the right hand side of
  \eqref{eqn:suff_const_dyn} is positive when $h_e \approx 0$, i.e. near the boundary
  of energy sufficiency safe set. The term
  $\tfrac{\mathcal{P}(u)}{\mathcal{P}(v_r)}v_r>0$ by definition, so the sign of
  the right hand side of \eqref{eqn:suff_const_dyn} depends on sign of $\dot{L}$. It can be
  shown that the sign of $\dot{L}(1-s)$ depends on the sign of
  $\tfrac{d}{dt}||w_1-w_2||$, therefore even if $\tfrac{d}{dt}||w_1-w_2||<0$, it
  will be so until $ ||w_1-w_2||\approx 0$, when the right hand side of
  \eqref{eqn:suff_const_dyn} will be positive. Therefore, when $h_e\approx 0$,
  $\eta > 0$, meaning $s$ will increase even when $0 < s < \epsilon_a$.
  
%  we consider that $\eta_{nom}=0$, which if violates\todo{Not English, not sure
%    what it means} \eqref{eqn:suff_const_dyn}\todo{Not sure this is the right
%    reference}, taking the equality of \eqref{eqn:suff_const_dyn} maintains
%  safety and this implies $\eta\geq0$ ($s$ increases and $x_r$ moves along the
%  path). 
  As a result, the reference point $x_r$ moves along the path and $h_d$ in
  \eqref{eqn:tracking_cbf} approaches zero (in the limit case $||w_1-x_r|| =
  d$ at $s=\bar{\epsilon}_a$). The fact that $s(t)$ is continuous and
  $\epsilon_a \ll \bar{\epsilon}_a$ implies that $\dot{w}_1 = 0$ before $s(t) =
  \bar{\epsilon}_a$, i.e. the path freezes while $h_d > 0$, so the path freezing
  condition \eqref{eqn:switch} does not violate the tracking CBF $h_d$, nor the energy
  sufficiency CBF $h_e$, and consequently the result of lemma \ref{lemma:base}
  follows (since $h_e\approx0$ and $\eta > 0$ implying $s>0$ leading to
  $h_d\approx0$).
  
%  making the path
%  static before $h_d$ approaches the boundary of its safe set (i.e.
%  $||x-x_r||\approx d$), thus the result of lemma \ref{lemma:base} follows.
\end{proof}

%\begin{remark}
%	The previous treatment was carried out considering $\eta_{nom}=0$ for
%  simplification, but the same conclusion is valid for $\eta_{nom} < 0$. The
%  main idea is to show that the right hand side of \eqref{eqn:suff_const_dyn}
%  becomes positive, for which the inequality will not hold for
%  $\eta_{nom}<0$. As the system goes closer to the boundary of the safe set
%  (i.e. $h_e \ll$), $\tfrac{\mathcal{P}(u)}{\mathcal{P}(v_r)}v_r>0$ by
%  definition, and it can be shown that the sign of $\dot{L}(1-s)$ depends on the
%  sign of $\tfrac{d}{dt}||w_1-w_2||$. Consequently even if $\dot{L}(1-s)<0$, it
%  will be so until $ ||w_1-w_2||\ll$ to have the right hand side of
%  \eqref{eqn:suff_const_dyn} greater than zero eventually\todo{Really tortured sentence...}.
%\end{remark}
%%% Local Variables:
%%% mode: latex
%%% TeX-master: "ijrr_draft"
%%% End:

%% file: corner_power_method.tex
The method described so far uses a single integrator model to describe robot
dynamics. Although such model choice is widely used in robotics and has the
advantage of versatility \citep{zhao2017defend}, applying it directly to more
specific robot models needs proper adaptation, especially considering the effects of
unmodelled modes of motion on power consumption. In this section we describe a
method to apply the proposed framework on a non-holonomic wheeled robot,
which has the added characteristic of being able to spin. More specifically we
are interested in robots with the following unicycle kinematic model
\begin{equation}
\begin{split}
\dot{x}_1&= v\cos\theta\\
\dot{x}_2&= v\sin\theta\\
\dot{\theta} &= \omega
\end{split}
\label{eqn:unicycle_kinematics}
\end{equation} 
where $x=\begin{bmatrix}
x_1 x_2
         \end{bmatrix}^T\in\mathbb{R}^2$ is robot's position, $\theta\in\mathbb{R}$ is its orientation, $v \in\mathbb{R}$ and $\omega\in\mathbb{R}$ are the linear and angular speeds respectively which act as inputs. A single integrator speed $u$ from \eqref{eqn:QP_suff} can be transformed to linear and angular speeds for a unicycle through the following relation \citep{ogren2001control}:
\begin{equation}
\begin{bmatrix}
v\\\omega
\end{bmatrix}=\begin{bmatrix}
1&0\\0&\tfrac{1}{\ell}
\end{bmatrix}\begin{bmatrix}
\cos\theta&\sin\theta\\-\sin\theta & \cos\theta
\end{bmatrix}u
\label{eqn:unicycle_transformation}
\end{equation}
where $\ell > 0$ is a distance from the robot's center to an imaginary handle
point. We also choose to be able to move backward when the value of $v$ becomes
negative by doing the following
\begin{equation}
\begin{split}
v^\prime &= v\\
\omega^\prime &= \omega \frac{v}{|v|}
\end{split}
\label{eqn:twoway}
\end{equation}
A robot described by \eqref{eqn:unicycle_transformation} consumes additional
power due to its angular speed $\omega$ in addition to what is consumed by its
linear speed $v$, which calls for augmenting \eqref{eqn:power_parabolic} with
additional terms
\begin{equation}
\begin{split}
\mathcal{P}_u(v,\omega) &= m_{u_0} + m_{u_1}|v| + m_{u_2}|v|^2 + m_{u_1}^\prime |\omega|  +m_{u_2}^\prime|\omega|^2\\
%&= m_0 + m_1||u||+ m_2||u||^2 + \Delta_\omega
\end{split}
\label{eqn:power_unicycle}
\end{equation} 
We note that in this power model we assume no direct coupling effects between linear and angular speeds on power consumption.
\begin{remark}
	\label{remark:power_model}
	The power model in \eqref{eqn:power_unicycle} is different from that in
  \eqref{eqn:power_parabolic} and we seek to establish a relation between the
  two. When a robot is moving in a straight line, then $\mathcal{P}_u(v,0) =
  \mathcal{P}(u)$. From \eqref{eqn:unicycle_transformation}
  $||u||=\sqrt{v^2+(\omega\ell)^2}$ so when $\omega \neq 0$, $v$ decreases for
  the same $||u||$, meaning $\mathcal{P}_u(v,0)\leq\mathcal{P}_u(v,\omega)$. In
  this case we either have $\mathcal{P}(u) > \mathcal{P}_u(v,\omega)$ meaning
  that turning has no contribution to power consumption, or $\mathcal{P}(u) \leq
  \mathcal{P}_u(v,\omega)$ which means turning has significant contribution in
  power consumption. We are interested in the latter case and we consider that
	\begin{equation}
	\mathcal{P}_u(v,\omega) \leq \mathcal{P}(u) + \Delta_\omega
	\end{equation}
	where $\Delta_\omega\in\mathbb{R}$ is the change in power due to rotation.
\end{remark}
Using a power model that only accounts for linear speed is akin to having the power
consumption due to $\omega$ as a disturbance power $\Delta_\omega$, which may
lead to instability as discussed in Remark~\ref{remark:stability}. A solution
to this issue is choosing a fairly slow return speed value $v_r$ so as to
increase the stability margin $\Delta_{p}^{*}$ in \eqref{eqn:stability_margin},
however this may impose undesirable limitations on performance.

Since the path we are using is essentially a piecewise linear path with
waypoints $w_i, i=1,\dots,n_w$, robot spinning will be mostly near these
waypoints when it is changing its direction of motion. The idea behind our
proposed adaptation is to add a certain amount of power $\tilde{\delta}$ to
$\mathcal{P}(v_r)$ in \eqref{eqn:mod_roots} near the path's waypoints so that
roots of \eqref{eqn:mod_roots} always exist. In other words, if we define
$\tilde{\mathcal{P}}(v_r)=\mathcal{P}(v_r)+\tilde{\delta}$, we ensure that roots
for the following equation always exist
\begin{equation}
  ||u|| = \frac{\bar{\mathcal{P}}(u)}{\tilde{\mathcal{P}}(v_r)}v_r = \frac{\mathcal{P}(u)+\Delta_\omega}{\mathcal{P}(v_r)+\tilde{\delta}}v_r
\end{equation}

We note that choosing $\tilde{\delta} > \Delta_\omega$ has the effect of slowing
down the robot (since $\bar{\mathcal{P}}(u) - \tilde{\mathcal{P}}(u)<0$, we have
a similar effect of having a negative disturbance power $\Delta_p$ in
\eqref{eqn:mod_roots}, which slows down the robot as discussed in
Remark~\ref{remark:stability}). Therefore, choosing a constant value for
$\tilde{\delta}$ such that $\tilde{\delta} > \Delta_\omega$ is equivalent to
choosing a lower value of $v_r$.

Instead of using a constant value for $\tilde{\delta}$, our approach is to use
double sigmoid functions to make the value of $\tilde{\delta} >0$ only near
waypoints $w_i$ and zero otherwise, meaning that $\tilde{\delta}$ is only
activated near waypoints, which can be described as:
\begin{equation}
\tilde{\delta}(s) = \sum_{i=1}^{n_w-1}P_i\tilde{\sigma}_i(s)
\label{eqn:slowing_down_power}
\end{equation}
where $P_i > 0$ is a conservative estimate of power consumption due to rotation
near waypoint $w_i$ and $\tilde{\sigma}_i(s)$ is defined as
\begin{equation}
\begin{split}
\tilde{\sigma}_i(s) &= \tilde{\sigma}_i^{r} \tilde{\sigma}_i^{f}\\
\tilde{\sigma}_i^{r}(s)&= \frac{1}{1+e^{-\tilde{\beta}(s-(s_i-\tfrac{1}{2}\phi))}}\\
\tilde{\sigma}_i^{f}(s)&= \frac{1}{1+e^{\tilde{\beta}(s-(s_i+\tfrac{1}{2}\phi))}}\\
\phi&=\frac{\tilde{d}}{L}
\end{split}
\end{equation}
where $\tilde{\beta} > 0$, and is $\tilde{d}$ a distance on the path from the
start of slowing down till its end and $L$ being the path length. The expression
\eqref{eqn:slowing_down_power} aims to start activating $\tilde{\delta}(s)$ a
distance $\tfrac{\tilde{d}}{2}$ before waypoint $w_i$ along the path, and end
this activation a distance $\tfrac{\tilde{d}}{2}$ after the waypoint along the
path. Figure~\ref{fig:slowed_paths} illustrates this idea for the path example
from Figure~\ref{fig:smooth_paths}. We update the energy sufficiency candidate
CBF in \eqref{eqn:energy_suff_cbf} to be

\begin{figure}[!htb]
	\centering
	 \includegraphics[trim=80 120 90 150,clip,width=\columnwidth]{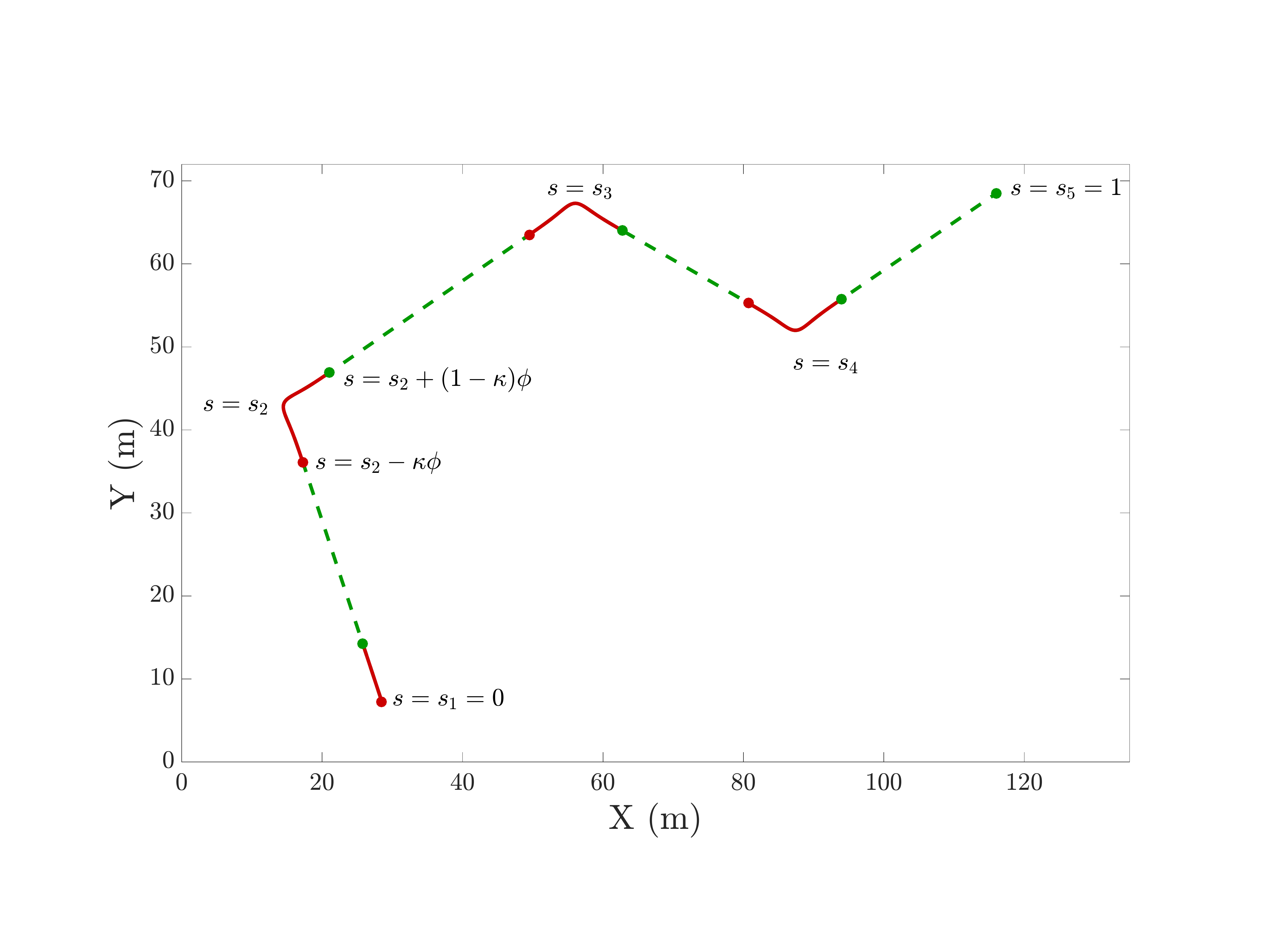}
	\caption{An example showing activation regions of $\tilde{\delta}(s)$. The path is similar to that illustrated in Figure~\ref{fig:smooth_paths}. The red segments are segments where $\tilde{\delta}(s)$ is activated. The red dots indicate the points where $s = s_i-\tfrac{1}{2}\phi$ and the green ones are points where $s = s_i+\tfrac{1}{2}\phi$. In this example we choose $\tilde{d}=15$m.}
	\label{fig:slowed_paths}
\end{figure}

%We choose $P_i$ in \eqref{eqn:slowing_down_power} to be a linear function in the angle between the two path segments intersecting at $w_i$. This angle can be calculated as 
%\begin{equation}
%\theta_{i}^{i+1} = \cos^{-1}\frac{\Delta w_{i-1}^{i}.\Delta w_{i}^{i+1}}{||\Delta w_{i-1}^{i}||.||\Delta w_{i}^{i+1}||}
%\end{equation}
%where $\Delta w_{i}^{i+1} = w_{i+1}-w_i$. The value of $P_i$ is then
%\begin{equation}
%P_i = \bar{P}\theta_{i}^{i+1}
%\end{equation}
%with $\bar{P} > 0$ being a maximum power consumption bound.

\begin{equation}
h_e = E_{nom} - E - \tfrac{\mathcal{P}(v_r)}{v_r}(L(1-s) - \delta_m) - \int_{s}^{1}\tilde{\delta}(\tau)d\tau
\label{eqn:he_mod}
\end{equation}
and applying Leibniz rule to the last term the constraint associated with this
candidate CBF is
\begin{equation}
-\mathcal{P}(u)-\Delta_\omega+\tfrac{\bar{\mathcal{P}}(v_r)}{v_r}(L\eta - \dot{L}(1-s)) + \tilde{\delta}(s)\eta\geq -\gamma_e h_e
\label{eqn:mod_he_const}
\end{equation}
We note that the integrand in \eqref{eqn:he_mod} can be carried out numerically.
Before showing that \eqref{eqn:he_mod} is a CBF, we need to choose appropriate
values for $P_i$ in \eqref{eqn:slowing_down_power} and $\tilde{d}$, both calling
for an estimate for a bound on rotation speed of a robot near a waypoint.
\begin{lemma}
	\label{lemma:max_omega}
	For a unicycle type robot with kinematics described in
  \eqref{eqn:unicycle_kinematics}, applying the transformation
  \eqref{eqn:unicycle_transformation} and \eqref{eqn:twoway} to follow a single
  integrator control input of a point moving with a speed $v_r$ along the path,
  the rotation speed at a waypoint $w_i$ is
	\begin{equation}
	\omega_i \leq \frac{v_r}{\ell}\sin\psi_i
	\end{equation}
	where $\psi$ is defined in \eqref{eqn:psi}.
\end{lemma}
\begin{proof}
	We start by considering $0\leq\psi \leq \tfrac{\pi}{2}$. Without loss of
  generality, suppose there is a robot at origin with $\theta = 0$ (aligned with
  x-axis), and $u = v_r\begin{bmatrix} \cos\psi& \sin\psi
	\end{bmatrix}^T$.  From \eqref{eqn:unicycle_transformation} we have
	\begin{equation}
	\omega = \frac{v_r}{\ell}\left(\cos\theta\sin\psi-\sin\theta\cos\psi\right)= \frac{v_r}{\ell}\sin(\psi-\theta)
	\label{eqn:omegaa}
	\end{equation}
	Consider the Lyapunov function $V=\alpha^2$ where $\alpha = \psi-\theta$, then
  $\dot{V} = -2\alpha\dot{\theta}=-2\tfrac{v_r}{\ell}\alpha\sin\alpha \leq 0$
  and since $\dot{V}=0$ only at $\alpha=0$ then $\alpha$ converges to $\alpha=0$
  by virtue of Lasalle's invariance principle, meaning that $\alpha$, and
  hence $\omega$, monotonically decrease for $\psi \leq \tfrac{\pi}{2}$. We are
  interested in finding the maximum value of $\omega$, so by differentiating
  \eqref{eqn:omegaa}
	\begin{equation}
	\frac{d\omega}{d\theta} = -\frac{v_r}{\ell}\cos(\psi-\theta) = 0 \Rightarrow \theta^* = \psi-\frac{\pi}{2}.
	\end{equation}
	Since we are considering $\psi\leq\tfrac{\pi}{2}$, it follows that $\theta^* \leq 0$.
  However, since we consider that the robot starts at $\theta = 0$ and that $\omega$
  monotonically decreases, then we
  can take $\theta^*=0$, meaning $\omega$ is maximum at $\theta=0$. Thus
	\begin{equation}
	\omega\leq\omega^*=\tfrac{v_r}{\ell}\sin\psi
	\label{eqn:semi_omega}
	\end{equation}
  If the robot is at waypoint $w_i$ pointed to the direction of vector $\Delta
  w_{i-1}^{i} = w_i-w_{i-1}$, there is always a rotation of axes from the global
  axes to new ones where $\theta = 0$ and translation of axes to place $w_i$ at
  the origin, and we reach the same result in \eqref{eqn:semi_omega}. The lemma
  follows by choosing $\psi = \psi_i$. We note that the same result holds for
  $-\tfrac{\pi}{2}\leq\psi<0$, but $\omega$ changes sign by virtue of
  \eqref{eqn:twoway}. For $ \tfrac{\pi}{2} < \psi < \pi$, from
  \eqref{eqn:unicycle_transformation} $v < 0$ and therefore $\omega^\prime =
  -\tfrac{v_r}{\ell}\sin(\psi-\theta)=\tfrac{v_r}{\ell}\sin((\psi-\pi)-\theta)$.
  Using the same procedure but for angle $(\psi-\pi) < 0$ we get the same
  result.
\end{proof}

We can use the upper bound estimate of rotation speed near a waypoint $w_i$ to
estimate an upper bound for the power consumed during a rotation $\Delta_\omega$
\begin{equation}
\Delta_\omega = \mathcal{P}_u(0,\tfrac{v_r}{\ell}\sin\psi)
\end{equation} 
In the following we estimate the activation distance $\tilde{d}$ near a waypoint $w_i$.

\begin{proposition}
	For a robot with model \eqref{eqn:unicycle_kinematics} and applying
  \eqref{eqn:unicycle_transformation} and \eqref{eqn:twoway} to follow a single
  integrator control input for a point moving with speed $v_r$ along the path,
  the distance needed till attenuation of angular speed, i.e.
  $\omega\leq\tfrac{v_r}{\ell}\epsilon_\omega$ with $\epsilon_\omega$ being an
  arbitrarily small number, is
	\begin{equation}
	d_a \leq \ell \frac{\pi}{2}\log\frac{\psi_i}{\epsilon_\omega}
	\end{equation}
\end{proposition}
\begin{proof}
	We use the candidate Lyapunov function $V = \alpha^2 = (\psi_i-\theta)^2$, so
  $\dot{V} = 2\tfrac{v_r}{\ell}\alpha\sin\alpha$. One result of
  \eqref{eqn:twoway} is that $\alpha\in[-\tfrac{\pi}{2},\tfrac{\pi}{2}]$ as
  discussed in proof of Lemma~\ref{lemma:max_omega}. We can prove exponential
  stability for the candidate Lyapunov function if we can find $k_1,k_2,k_3 > 0$
  such that \citep{khalil2002nonlinear}
	\begin{equation}
	\begin{split}
	&k_1\alpha^2 \leq V \leq k_2\alpha_2\\
	&\dot{V}\leq -k_3\alpha^2
	\end{split}
	\end{equation}
	Since $V=\alpha^2$, then $k_1=k_2=1$. We can estimate $k_3$ by letting the
  parabola $f(\alpha)=k_3\alpha^2$ and
  $g(\alpha)=\tfrac{2v_r}{\ell}\alpha\sin\alpha$ intersect at $\alpha =
  \tfrac{\pi}{2}$, which gives $k_3 = \tfrac{4v_r}{\ell\pi}$. By virtue of $V$
  being exponentially stable on $\alpha\in[-\tfrac{\pi}{2},\tfrac{\pi}{2}]$,
  then
	\begin{equation}
	\alpha\leq \psi_i e^{-\tfrac{k_3}{2}t}
	\end{equation}
	then at time $\tilde{t}$ the right hand side of the last inequality is equal
  to $\epsilon_\omega$
	\begin{equation}
	\psi_i e^{-\tfrac{k_3}{2}\tilde{t}}=\epsilon_\omega\Rightarrow \tilde{t}=\tfrac{2}{k_3}\log\tfrac{\psi_i}{\epsilon_\omega}
	\end{equation}
	and we note that $\alpha<\epsilon_\omega$ at $t>\tilde{t}$. The attenuation
  distance then is
	\begin{equation}
	d_a\leq\tilde{d}_a=\tilde{t}v_r=\ell \frac{\pi}{2}\log\frac{\psi_i}{\epsilon_\omega}
	\label{eqn:attenuation_dist}
	\end{equation}
	We note that at $t=\tilde{t}$ the angular speed will be 
	\begin{equation}
	\omega \leq \frac{v_r}{\ell}\sin\epsilon_\omega\approx\frac{v_r}{\ell}\epsilon_\omega
	\end{equation}
\end{proof}

\begin{theorem}
	For a robot with unicycle kinematics, applying
  \eqref{eqn:unicycle_transformation} and \eqref{eqn:twoway}, and provided that
  $\tilde{\delta}$ in \eqref{eqn:slowing_down_power} is formed such that
%	\begin{equation}
%	P_i \geq \mathcal{P}_u(0,\tfrac{v_r}{\ell}\sin\psi_i)\frac{L}{v_r}
%	\end{equation} 
	\begin{equation}
	\tilde{\delta} \geq\frac{L}{v_r}\left(\sum_{i=1}^{n_w-1}\mathcal{P}_u(0,\tfrac{v_r}{\ell}\sin\psi_i)\tilde{\sigma}(s)+\Delta_\epsilon\right)
	\end{equation}
	and $\tilde{d}=2\max\{\tilde{d}_a,d\}$, where $d$ is the tracking distance
  from \eqref{eqn:tracking_cbf} and $\Delta_\epsilon$ is robot's power
  consumption when $\omega=\epsilon_\omega$, and if $v_{r}^{*} =
  \sqrt{\frac{m_0}{m_2}}\leq u_{\text{max}} $, with $u_{max}>0$ being maximum robot speed, then $h_e$ in \eqref{eqn:he_mod}
  is a ZCBF. Moreover, if \eqref{eqn:mod_he_const} is applied in
  \eqref{eqn:QP_suff} instead of \eqref{eqn:suff_constraint}, energy sufficiency
  is guaranteed.
\end{theorem}

\begin{proof}
	We start by considering when the robot moves on a straight line and away from
  waypoints, i.e. $s_{i-1}+\tfrac{\phi}{2} < s < s_i-\tfrac{\phi}{2}$, in which
  case the proof is similar to proof of lemma~\ref{lemma:base} since
  $\mathcal{P}(u) = \mathcal{P}_u(v,0)$ as pointed out in
  remark~\ref{remark:power_model}.
	
	We note that when $h_d\approx 0$, the robot is following the reference point $x_r$ along a straight
  line and it starts rotating after $x_r$ passes waypoint $w_i$, i.e. $s\geq
  s_i$, and since $||x-x_r||\approx d$ it means that the robot will start
  rotation a distance $d$ away from $w_i$, but since
  $\tilde{d}=2\max\{\tilde{d}_a,d\}$, $\tilde{\delta}$ will be activated at a
  distance greater than or equal $d$ before $w_i$ along the path, i.e. before
  the robot starts spinning. Also when $s_i <s \leq s_i+\tfrac{\phi}{2}$ and due
  to choice of $\tilde{d}$, $s = s_i+\tfrac{\phi}{2}$ happens at least a
  distance $\tilde{d}_a$ after $w_i$ along the path and at this point $\omega <
  \tfrac{v_r}{\ell}\epsilon_\omega$, i.e. $\tilde{\delta}$ will be deactivated
  after $\omega$ has been attenuated.
	
	When $\tilde{\delta}$ is activated, i.e. $s_{i}-\tfrac{\phi}{2} < s <
  s_i+\tfrac{\phi}{2}$, and considering the critical case where $h_e \approx 0$
  and $h_d\approx 0$, similar to what we did in the proof of
  Lemma~\ref{lemma:base}, we consider the equality of
  \eqref{eqn:tracking_constraint} and \eqref{eqn:mod_he_const}, so from
  \eqref{eqn:mod_he_const} (and noting that $\dot{L}=0$ when the robot is
  moving along the path by virtue of proposition~\ref{proposition:freeze})
	\begin{equation}
	\eta = \frac{\mathcal{P}(u)+\Delta_\omega}{\mathcal{P}(v_r)+\tilde{\delta}\frac{v_r}{L}}\frac{v_r}{L}
	\end{equation}
	and doing the same steps to obtain \eqref{eqn:abs_speed}
	\begin{equation}
	||u||=\frac{\mathcal{P}(u)+\Delta_\omega}{\mathcal{P}(v_r)+\tilde{\delta}\frac{v_r}{L}}v_r
	\label{eqn:mod_abs_speed}
	\end{equation}
	which is a similar root finding problem to \eqref{eqn:abs_speed}. Since
  $\tilde{\delta}\geq
  \mathcal{P}_u(0,\tfrac{v_r}{\ell}\sin\psi_i)\frac{L}{v_r}$, then
  $\tilde{\delta}\tfrac{v_r}{L}\geq\Delta_\omega$, which ensures roots for
  \eqref{eqn:mod_abs_speed} exist and that $||u||$ will converge to a slower
  speed than $v_r$ as discussed in Remark~\ref{remark:stability} (since having
  $\tilde{\delta}\tfrac{v_r}{L}-\Delta_\omega < 0$ has a similar effect as
  having $\Delta_p<0$ in Remark~\ref{remark:stability}). Moreover, when
  $\omega<\tfrac{v_r}{\ell}\epsilon_\omega$ , \eqref{eqn:mod_abs_speed} will
  become
	\begin{equation}
	||u||=\frac{\mathcal{P}(u)+\Delta_\epsilon}{\mathcal{P}(v_r)+\tilde{\delta}\frac{v_r}{L}}v_r
	\label{eqn:mod_abs_speed_2}
	\end{equation}
	which is guaranteed to have roots since
  $\tilde{\delta}\tfrac{v_r}{L}\geq\Delta_\epsilon$. Thus provided that
  $\sqrt{\frac{m_0}{m_2}}\leq u_{\text{max}}$ there is always a value of $u$
  that satisfies \eqref{eqn:mod_he_const}. Since \eqref{eqn:tracking_cbf} and
  \eqref{eqn:he_mod} are ZCBFs and if $\delta$ in \eqref{eqn:he_mod} is equal to
  $\delta_m$ from \eqref{eqn:delta_m} then the robot's energy satisfies $E(t) =
  E_{nom}$ only when $||x-x_c|| < \delta$ as discussed in the proof of
  Theorem~\ref{thm:suff}, thus ensuring energy sufficiency.
\end{proof}

We can apply the same quadratic program in \eqref{eqn:QP_suff}, but with
replacing the definition of energy sufficiency CBF, and for that the $A$ and $B$
matrices in \eqref{eqn:QP_suff} will be
\begin{equation}
\begin{split}
\mathbf{A} &= \begin{bmatrix}
\tfrac{\mathcal{P}(v_r)}{v_r}L+\tilde{\delta} &\mathbf{0}_{1\times 2}\\1&\mathbf{0}_{1\times 2}\\
(x-x_r)^T\tfrac{\partial x_r}{\partial s}&-(x-x_r)^T
\end{bmatrix}\\ 
\mathbf{B} &= \begin{bmatrix}
-\gamma_e h_e + \mathcal{P}(u)+\Delta_\omega+\dot{L}(1-s)\\-\gamma_bh_b\\
-\gamma_dh_d
\end{bmatrix}\\
\mathbf{u}_{nom}&=\begin{bmatrix}
0&u_{nom}
\end{bmatrix}
\end{split}
\label{eqn:defs_uni}
\end{equation}
We can follow the same steps as in Theorem~\ref{thm:suff} to show that energy
sufficiency is maintained solving this QP problem over a fixed path, with the same
path freezing idea as in Proposition~\ref{proposition:freeze}.

The treatment thus far concerns a unicycle robot moving around, with a fixed
path back to charging station. A path planner could be used to update the path
in the same manner discussed in Section~\ref{sec:dynamic}. We can use a similar
sequence as in Algorithm~\ref{algo}, but we need to show that the SPC method is
a valid backup for the proposed unicycle adaptation.

\begin{proposition}
  Sequential Path Construction does not violate energy sufficiency for a robot
  described by \eqref{eqn:unicycle_kinematics} when applying \eqref{eqn:QP_suff}
  with transformation \eqref{eqn:unicycle_transformation}, \eqref{eqn:twoway},
  and with $A$ and $B$ matrices described in \eqref{eqn:defs_uni}.
\end{proposition}
\begin{proof}
	Similar to proposition~\ref{proposition:spc}, provided that there exists a
  value of $u=u_{nom}$ and $\eta=\eta_{nom}$ satisfying following inequalities
  \small
	\begin{subequations}
		\begin{equation}
		h_e = E_{nom} - E - \tfrac{\mathcal{P}(v_r)}{v_r}(L(1-s) - \delta) - \int_{s}^{1}\tilde{\delta}(\tau)d\tau
		\label{eqn:he_mod_}
		\end{equation} 
		\begin{equation}
		-\mathcal{P}(u)-\Delta_\omega+\tfrac{\bar{\mathcal{P}}(v_r)}{v_r}(L\eta_{nom} - \dot{L}(1-s)) + \tilde{\delta}(s)\eta_{nom}\geq -\gamma_e h_e
		\label{eqn:mod_he_const_}
		\end{equation}
		\label{eqn:uni_spc_cond}
	\end{subequations}
	\normalsize
	
	we need to show that \eqref{eqn:uni_spc_cond} is not violated at a path
  update. Similar to proof of Proposition~\ref{proposition:spc}, provided that
  nominal control inputs are continuous and satisfying \eqref{eqn:uni_spc_cond},
  then there are no jumps (i.e. instantaneous changes) for
  $E,\mathcal{P}(u),\Delta_\omega$ and $v_r$. Moreover since SPC is path length
  invariant, $L$ does not change. Also since SPC is path angle invariant from
  Lemma~\ref{lemma:length_invariance}, then the increase in power due to the
  addition of the new waypoint is equal to zero (because $\psi_2$ for the new
  path is equal to zero) and no change occurs for the power consumption along
  the path, therefore $\tilde{\delta}$ does not jump as well, meaning
  \eqref{eqn:uni_spc_cond} is not violated under SPC.
\end{proof}

We can apply Algorithm~\ref{algo} and the same logic in
Theorem~\ref{thm:SPC_suff} to show that energy sufficiency is maintained under
discrete path updates.

\begin{remark}
	The adaptation we are using for the method based on
  single integrator dynamics (in Section~\ref{sec:dynamic}) to unicycle dynamics
  is versatile and can go beyond accounting for excess power consumption near
  waypoints. This is due to the fact that the estimated
  excess power, e.g. \eqref{eqn:slowing_down_power}, is modelled as a summation
  of double sigmoid functions, activated along different segments of the path.
  Moreover, this excess in the estimated power consumption is incorporated in
  the energy sufficiency CBF \eqref{eqn:he_mod} through numerical integration,
  making it easier to account for different types of ``resistance'' along the
  path. For example, effects like surface inclinations, variability in friction
  and increased processing power, among many others, could be modelled in a
  similar way to \eqref{eqn:slowing_down_power} through identifying ranges of
  the path parameter $s$ corresponding to different segments on the path, each
  associated with a double sigmoid function multiplied by the estimated power
  consumption related to the effect being modelled. This idea allows to adapt
  the methods based on single integrator dynamics to a wide range of scenarios,
  environments and robot types.
\end{remark}
%\begin{remark}
%	We point out the versatility of the idea of augmenting estimated robot's power consumption with additional terms that depend on the its position on the path, similar to what we do in \eqref{eqn:slowing_down_power} to account for increased power consumption near waypoints. 
%\end{remark}
%%% Local Variables:
%%% mode: latex
%%% TeX-master: "ijrr_draft"
%%% End:

%% file: result_energy.tex
\subsection{Simulation Setup}
We present the simulation results that highlight the ability of our proposed
framework to ensure energy sufficiency during an exploration mission. The
considered experimental scenario allows the robots to perform the exploration
mission while ensuring robot's energy consumption is within the dedicated energy
budget $E_{nom}$.
 
We evaluate the approach using a physics-based
simulator~\citep{pinciroli2012argos}. We use a simulated KheperaIV robot
equipped with a 2D lidar with a field of view of 210 degrees and a 4m range as
the primary perception sensor. The architecture of the autonomy software used in
simulations is shown in Figure~\ref{fig:arc}A. The autonomy software allows the
robot to explore and map the environment. Each robot is equipped with a
volumetric mapping system \citep[Voxblox]{oleynikova2017voxblox} using Truncated
Signed Distance Fields to map the environment. A graph-based exploration planner
\citep[GBPlanner,][]{dang2019graph} uses the mapping system to plan both the
exploration and homing trajectories. We carry out a path shortening procedure as
described in \citep[Algorithm 1]{cimurs2017bezier} to eliminate redundant and
unnecessary points from the original path planner output, making the final path
straighter and shorter. Our proposed framework is implemented as a Buzz~\citep{pinciroli2016buzz}
script that periodically queries the exploration planner for a
path and applies the required control commands to the robot.

\begin{figure*}[!htb]
	\centering
	\includegraphics[scale=0.68]{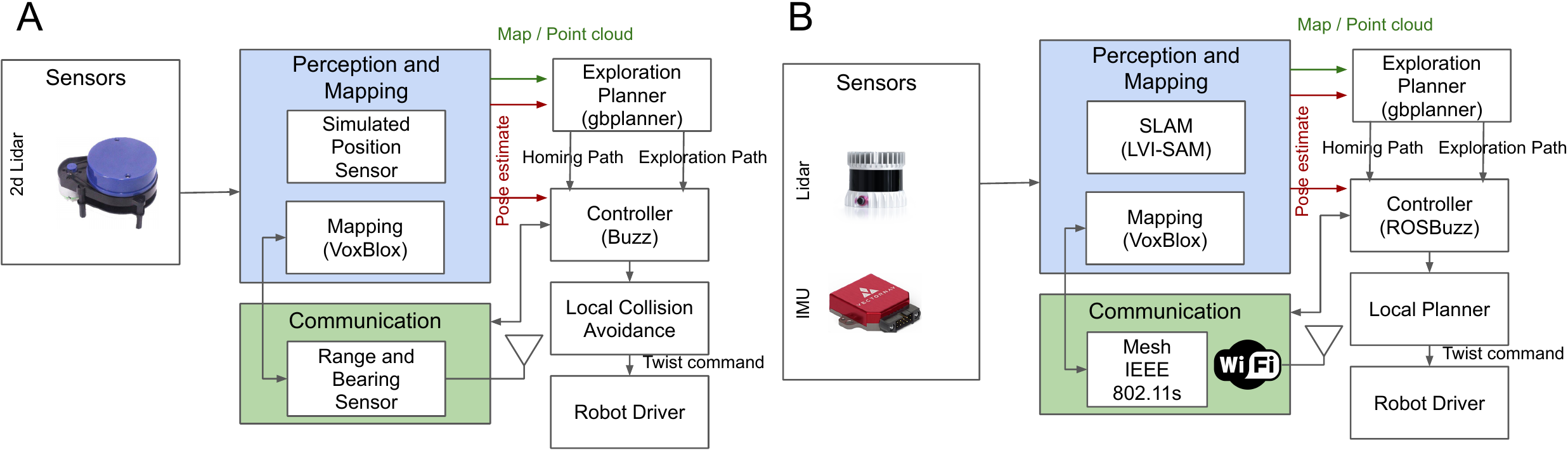}
	\caption{Software Architecture used during the simulation study (A) and on the experimental hardware (B).\label{fig:arc}}
%\citep{sturtevant2012benchmarks}.}
\end{figure*}

We use the maze map benchmarks from \citep{sturtevant2012benchmarks} as a
blueprint for obstacles in the environment, and each map is scaled so that it
fits a square area of $30\times 30$ meters. In each simulation one robot maps
the unexplored portions of the map to maximize its volumetric gain
\citep{dang2019graph}. We run four groups of experiments for three different
maps from the benchmarking dataset~\citep{sturtevant2012benchmarks}, and for
each case we run 50 simulations with randomized configurations to obtain a
statistically valid dataset.
\begin{figure*}
	\centering
	\includegraphics{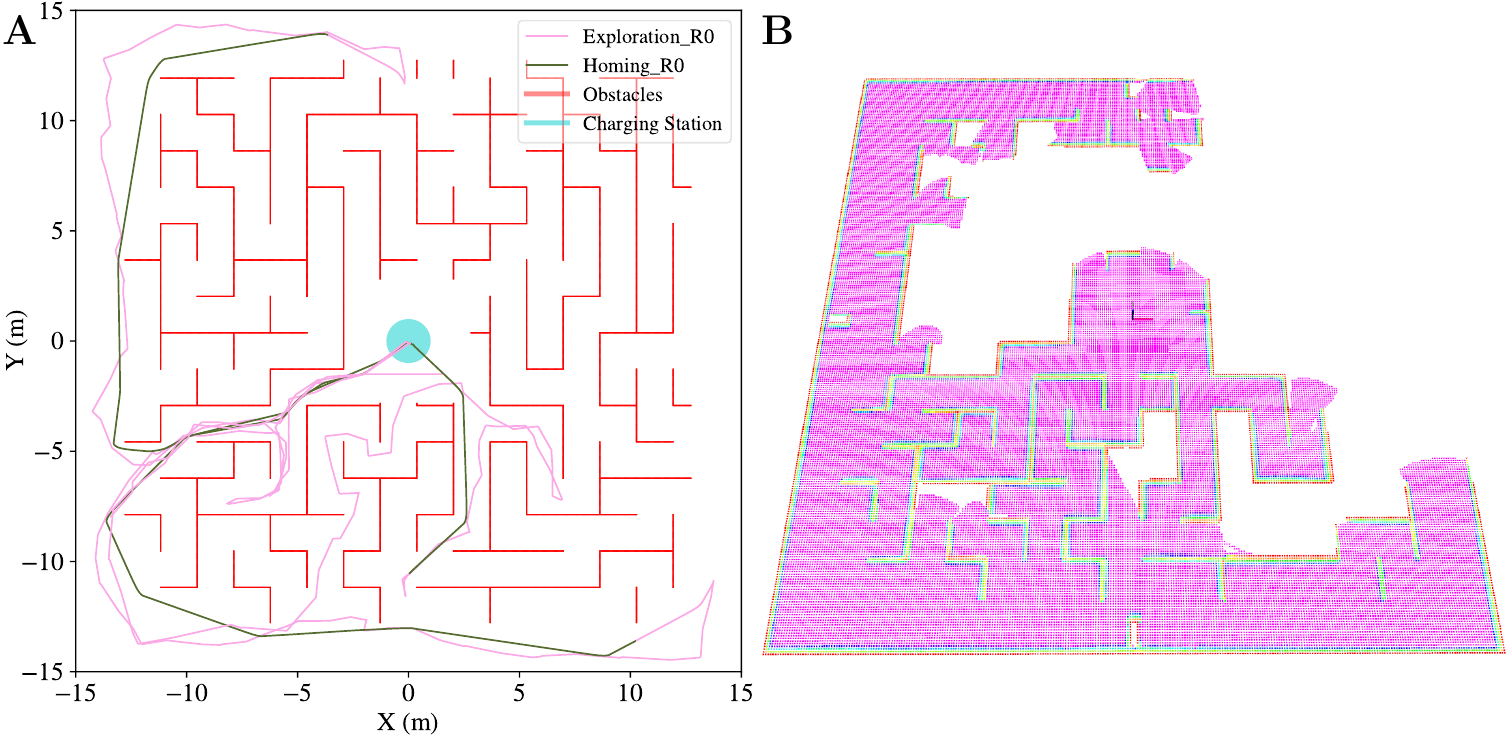}
	\caption{A sample result for the trajectories generated by the robot (A) and the map constructed in the same simulation run (B) for an exploration task in a maze environment (maze-4 from \citep{sturtevant2012benchmarks}), while using our proposed approach to maintain energy sufficiency.
    \label{fig:traj}}
%\citep{sturtevant2012benchmarks}.}
\end{figure*}

We use a polynomial power model to describe power consumption of the robots in
simulation. We derive this model by collecting power consumption readings from a
physical AgileX Scout Mini~\citep{scout} robot at different values of linear
and angular speed, then we fit a surface through these readings to obtain out power
model. Figure~\ref{fig:power} shows the fitted power model, along with the
actual collected power readings from the robot. We interface the single
integrator output $\mathbf{u}^*$ of \eqref{eqn:QP_suff} to the unicycle model of
the robots using the transformation \eqref{eqn:unicycle_transformation} and the
modification \eqref{eqn:twoway}. We use the robot's linear and angular speeds to
estimate the robot's power according to the following polynomial
\begin{equation}
\begin{split}
\mathcal{P}_u(v,\omega) &= 27.8126||v||^2 - 107.7343|\omega|^2\\& + 31.4578||v|| + 179.9095|\omega| + 1.234
\end{split}
\label{eqn:power_model}
\end{equation}
and the power model is depicted in Figure~\ref{fig:power}. We also add to this
model an additional power of $\mathcal{P}_{payload}=20$W to account for payload
power consumption.
%We note that if $\omega = 0$ in \eqref{eqn:power_model} the power will be a second order power model as proposed in \eqref{eqn:power_parabolic}. The presence of $\omega$ in this relation accounts for the effect of robot spinning on power consumption but it serves as an algebraically added effect to that of speed, and this poses no problem as long as conditions in lemma~\ref{lemma:base} and remark~\ref{remark:speed_power_rel} are met.
\begin{figure}[!h]
	\centering
	\includegraphics[width=1.\columnwidth]{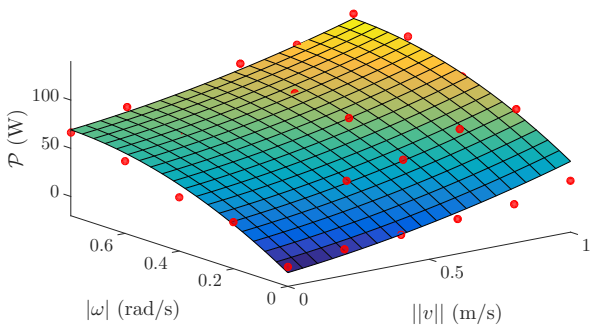}
	\caption{Surface plot of the power model used in simulation. The red dots are the actual measured power values at different values of linear and angular speeds ($v$ and $\omega$) and is fitted by a 3 dimensional surface to minimize the mean least square error between the model and the real data points.}
	\label{fig:power}
\end{figure}

\subsection{Simulation Results}

Figure~\ref{fig:traj} shows an example of the robots' trajectory during
exploration and returning to the charging station for one simulation run in one
of the maze environments (maze-4). Figure~\ref{fig:traj} also shows the map
built by the robot during this simulation run.
%Periodic broadcasts unify the maps of the robots, and hence only the map of a single robot is presented. 
As observed in the trajectory plot, any given robot's exploration trajectory is
always accompanied by a homing trajectory to the charging station satisfying the
energy constraints.

For all simulation runs we measure the estimated Total Area Covered (TAC) and
the Energy On Arrival (EOA), which is the amount of energy consumed by the robot
by the time it arrives back to the station, and we use these values as metrics
for performance. The TAC serves as a measure of the mission execution quality,
and the EOA is a measure of the extent the available energy budget has been
used. We run all test cases at two desired values of return speed: a slow speed
of $v_r=0.1$m/s and a faster speed of $v_r=0.5$m/s. We highlight the efficacy of
our approach by comparing the aforementioned metrics to the results of a
baseline method in which a robot returns back on the path when the available
energy reaches a certain fixed threshold percentage of the total nominal energy
(as it is a standard procedure with commercial robots). For the baseline we use
only the tracking CBF \eqref{eqn:tracking_cbf} in a QP problem similar to
\eqref{eqn:QP_suff} and we change the path parameter $s$ according to the
following relation
\begin{equation}
\dot{s} = \begin{cases}
\frac{v_r}{L},\quad \text{if} \quad \tfrac{E}{E_{nom}} > \tau\\
0,\quad \text{otherwise}
\end{cases}
\end{equation}
where $0 < \tau < 1$ is a threshold return energy ratio, and $L$ is the total
path length at the point when the robot starts moving back towards the charging
station. We show the results of our comparison for the aggregated
values of TAC for different simulation scenarios in Figure~\ref{fig:box_area}.
To highlight the relation between TAC and EOA we use red dots in
Figure~\ref{fig:box_area} to indicate TAC values corresponding to simulation
runs during which the energy budget is violated, i.e. EOA is less than zero at
least once, indicating the robot's failure to recharge before its energy budget is
fully consumed. Figure~\ref{fig:box_eoa} and ~\ref{fig:box_eoa2} show histograms
of EOA values distribution for our method as well as baseline at different
values of $\tau$ for $v_r=0.5$m/s and $v_r=0.1$m/s respectively. In all
simulation runs the total energy budget is set to be $12$kJ.

\begin{figure*}[!htb]
	\centering
	\begin{subfigure}[b]{0.9\textwidth}
		\centering
		\hspace*{-2cm}
		\includegraphics[width = \columnwidth]{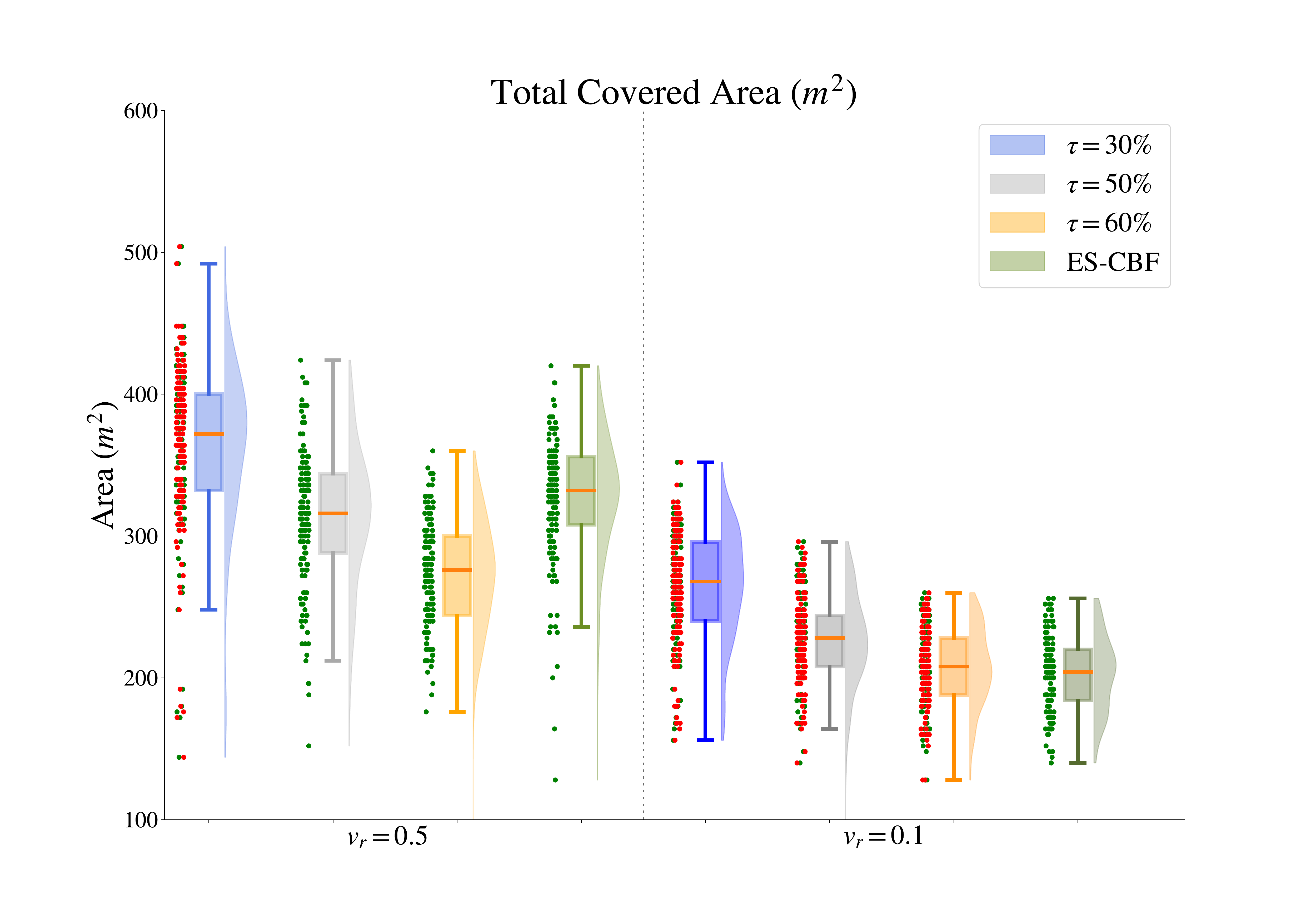}
		\caption{Estimated total area covered}
		\label{fig:box_area}
	\end{subfigure}
	\begin{subfigure}[b]{0.49\textwidth}
		\centering
		\hspace*{-2cm}
		\includegraphics[width = \columnwidth]{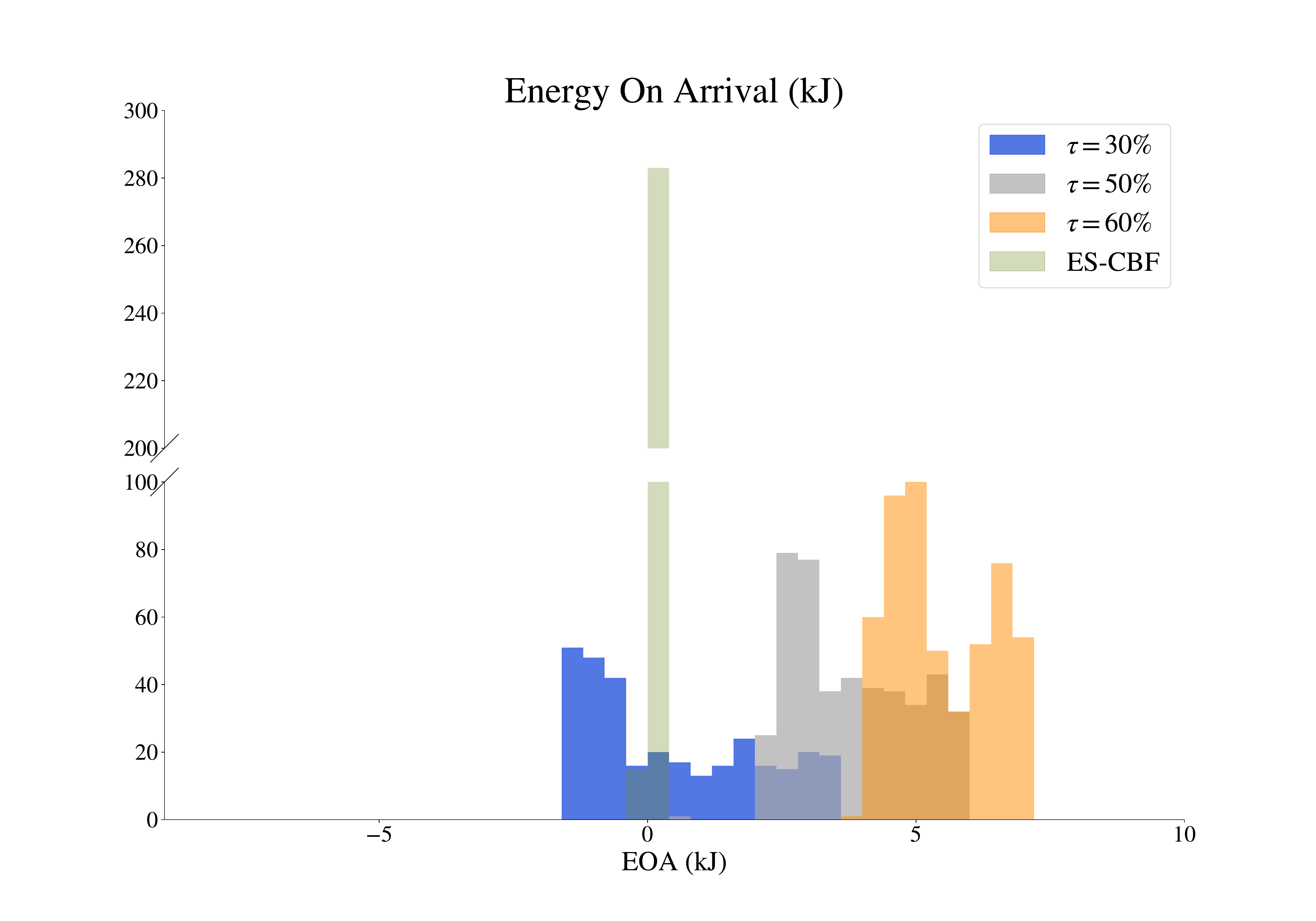}
		\caption{Energy on arrival for simulations with $v_r=0.5$m/s.}
		\label{fig:box_eoa}
	\end{subfigure}
	\begin{subfigure}[b]{0.49\textwidth}
		\centering
		\hspace*{-2cm}
		\includegraphics[width = \columnwidth]{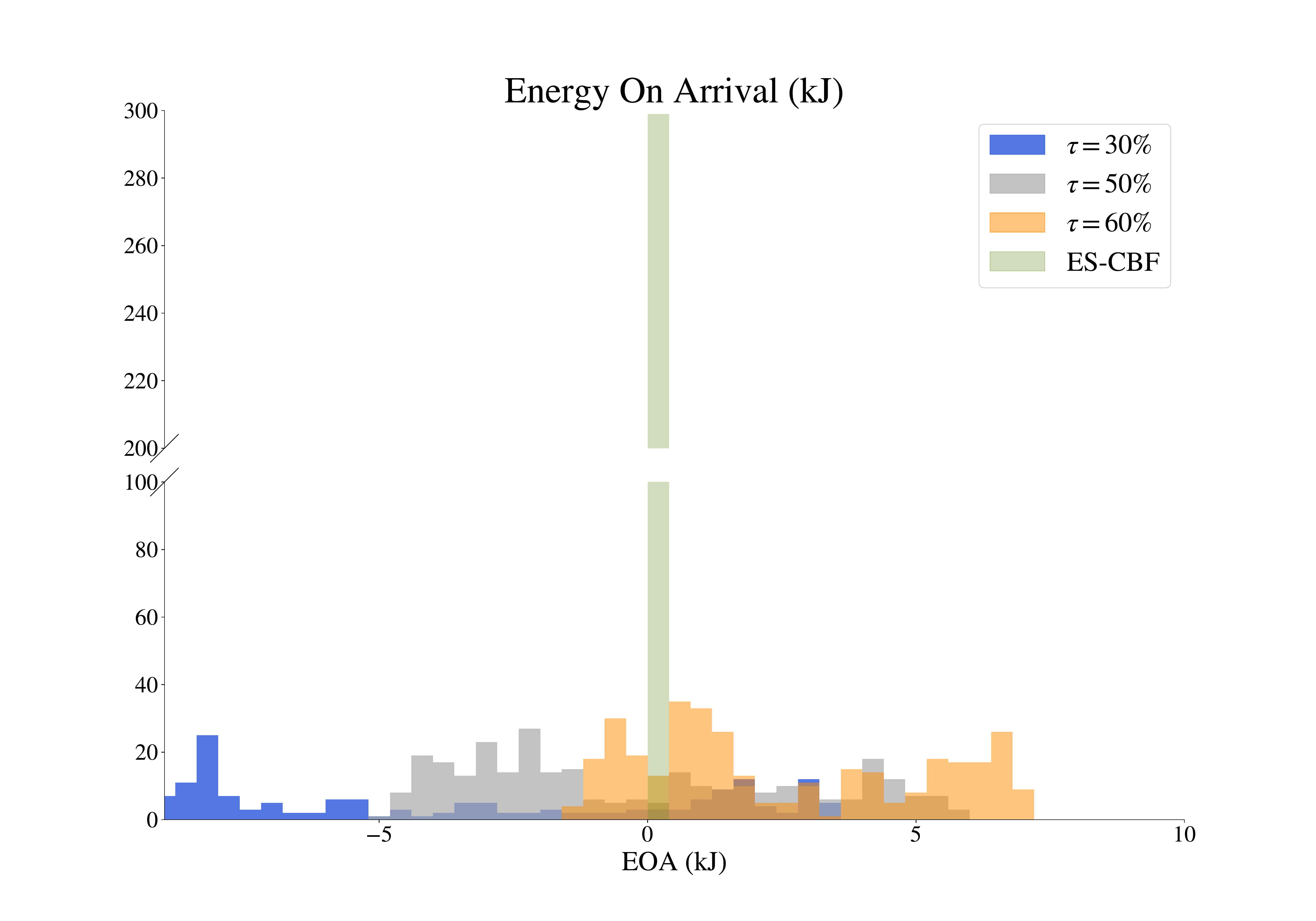}
		\caption{Energy on arrival for simulations with $v_r=0.1$m/s.}
		\label{fig:box_eoa2}
	\end{subfigure}
	\caption{Comparison between baseline method for three different threshold percentages $\tau$ and our CBF-based approach for energy sufficiency, denoted \emph{ES-CBF}. Simulation data for total area covered and energy values upon arrival to charging station is collected for three test environments and two different desired return speeds ($v_r=0.5$m/s and $v_r=0.1$m/s), each run for 50 instances with different random seeds. The red dots in Figure~\ref{fig:box_area} indicate area values corresponding to simulation instances where the energy budget is violated at least once, while green dots indicate no violation of energy budget. Histograms ~\ref{fig:box_eoa} and ~\ref{fig:box_eoa2} show distribution of energy on arrival (EOA) values for $v_r=0.5$m/s and $v_r=0.1$m/s respectively.}
%	\caption{Box plots comparing between baseline method for two different threshold percentages and our CBF based approach for different test environments. The box plots are color coded to differentiate between scenarios, and different shades of color as well as hatching patters indicate the method tested, i.e. baseline or CBF based energy sufficiency.}
	\label{fig:box}
\end{figure*} 

We note that in Figure~\ref{fig:box_area} for $v_r = 0.1$m/s the area covered
consistently increases with decreasing return energy threshold percentages, i.e.
when a robot starts returning to recharge at $\tau=0.3$ it typically covers more
area than when it needs to return at $\tau=0.5$ as it uses more of its energy to
carry out its mission. Although for this case the area covered using our
proposed method is less than baseline (box plot median value of 268m$^2$ for
$\tau=0.3$ and 204m$^2$ for ES-CBF, meaning a 24\% reduction in TAC in the worst
case), baseline results have significantly more red dots than ES-CBF, indicating
significantly more violations of energy budget than ES-CBF, so although TAC is
more for baseline the energy budget is violated for most test runs.

For $v_r = 0.5$m/s, Figure~\ref{fig:box_area} shows an overall increase in TAC
for both baseline and ES-CBF compared to the case where $v_r = 0.1$m/s.
Moreover, we notice an increase in TAC in case of ES-CBF over baseline with
$\tau=0.5$ and $\tau=0.6$ (5\% and 20\% increase in TAC respectively), while
there is a decrease of 10\% in TAC between baseline with $\tau=0.3$ and ES-CBF.
For baseline cases with $\tau=0.5$ and $\tau=0.6$ there are no red dots at
$v_r$=0.5m/s in Figure~\ref{fig:box_area}, but there are numerous violations of
energy sufficiency for baseline with $\tau=0.3$. Overall, choosing a threshold
value to return to the charging station depends on the map and task at hand, and
does not provide guarantees for either optimal mission success or return or
respecting the energy budget. On the contrary, our method guarantees that the
energy budget is fully exploited, without affecting mission performance.

It is also worth noting from Figure~\ref{fig:box_eoa} and \ref{fig:box_eoa2}
that the distribution of EOA values is very tight around zero, meaning that
robots applying ES-CBF framework arrive to the station without violating the
energy budget allocated and without wasting energy, i.e. robots do not arrive
too late or too early, which means full utilization of the energy allocated. On
the other hand for the baseline method we can see in Figure~\ref{fig:box_eoa2}
for $v_r=0.1$ that EOA values are more widely dispersed around zero, with a
significant portion of the values being positive or negative, indicating robots
arriving to station either too early or too late, which is a direct result of
not considering needed energy to return back to station (e.g. a robot could
reach the return threshold $\tau$ when it is relatively close to the station so
it will eventually arrive back with a significant amount of energy, and it may
reach $\tau$ when it is far away so that the energy budget is fully depleted on
the way back). We also note that for $v_r=0.5$m/s the values of EOA are mostly
positive for baseline with $\tau=0.5$ and $\tau=0.6$ indicating significant
non-utilized energy when the robot returns back to recharge. Therefore for these
two baseline cases the robots utilize less energy for exploration and this
explains the advantage that ES-CBF has in TAC over these two baseline cases at
$v_r=0.5$m/s.

% On the other hand we can see from Figure~\ref{fig:box_eoa} that since the basline method does not take into account the energy needed to return back to recharge, the energy consumed on arrival to station is typically significantly more than the allocated energy budget. For our method the energy on arrival is negligible in comparison to the baseline, which shows the efficacy of our method.

\subsection{Hardware setup}
We study the performance of the energy-sufficiency approach using an AgileX
Scout Mini rover equipped with a mission payload to perform exploration and
mapping missions, shown in Figure~\ref{fig:rover_setup}. The robot has an Ouster
OS0-64 lidar as the primary perception sensor and a high-performance Inertial
Measurement Unit (IMU) from VectorNav. A mesh communication router implements
IEEE802.11s to communicate with the base station.

\begin{figure}[!htb]
	\centering
	\includegraphics[width=0.8\columnwidth]{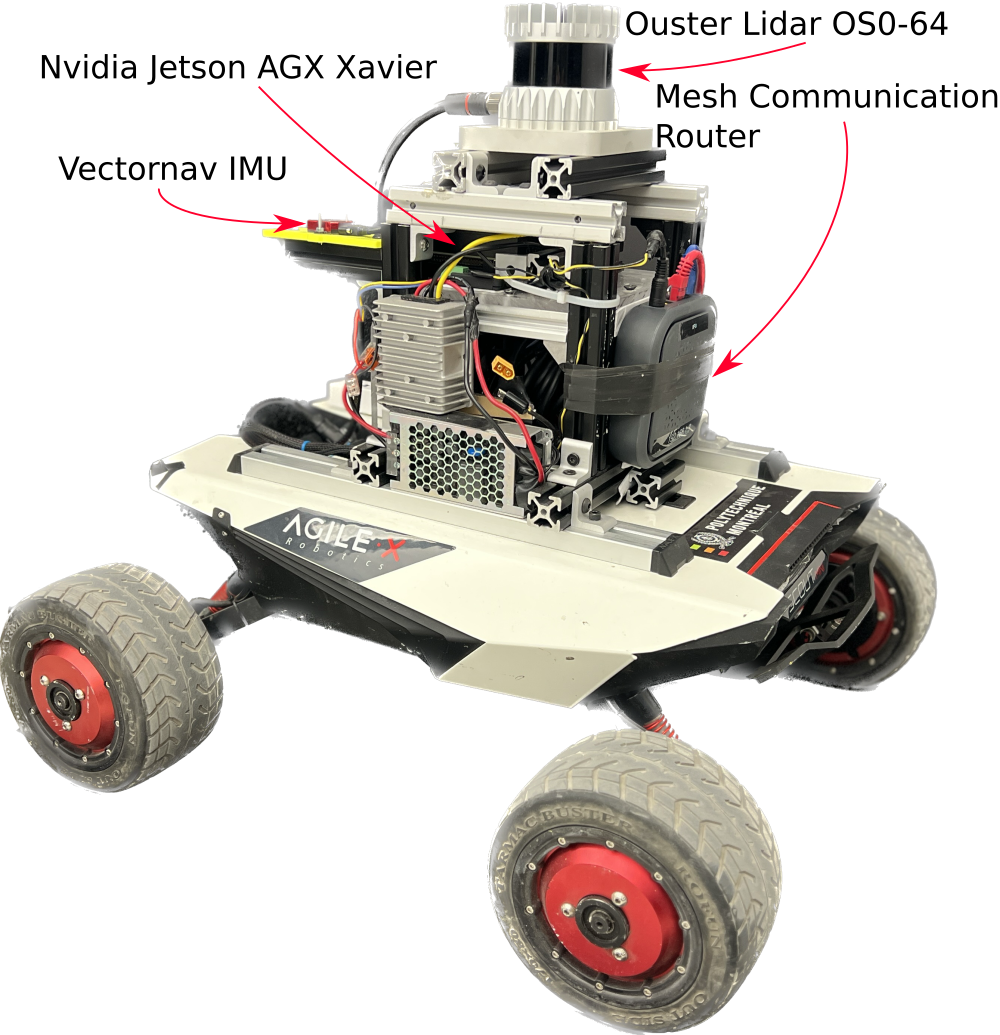}
	\caption{Experimental setup we use to perform the exploration mission while
    maintaining energy sufficiency. It consists of an AgileX Scout Mini rover with a mission payload mounted on top as demonstrated above.}
	\label{fig:rover_setup}
\end{figure}

Figure~\ref{fig:arc}B shows the software architecture deployed on the Nvidia
Jetson AGX Xavier of the rover. We implement a full stack Simultaneous
Localization And Mapping (SLAM) system, mesh communication system, and a local
planner for collision avoidance. Unlike the simulation, the rover performs a
full-stack 3D localization and mapping using a variant of
LVI\_SAM~\citep{lvisam2021shan} with a front-end generating pose graphs and a
back-end performing map optimization. The mapping \citep[Voxblox]{oleynikova2017voxblox} and planning
(Gbplanner\citation) modules and controller were the same for both simulation and
hardware. We apply the path shortening procedure in \citep[Algorithm
1]{cimurs2017bezier} on the output of the path planner, as we do in the
simulation setup.

We estimate the robot's power consumption using voltage and electric current
values. Upon arrival to charging region the robot executes a simple docking
manoeuvre to enter the charging region and carries out a simulated battery swap
operation to replenish the robot's energy. We point out that such setup does not
affect the validity of the experiment and could be justified by the fact that
the energy consumed by the robot is consistent, meaning that the power needed to
move the robot at a certain speed does not depend on the battery, but rather
depends on the robot's mechanical properties and the environment which are both
static.
%\todo[inline]{Hassan: Talk about how the energy consumption and battery is handled on the rover. As we do in the simulation setup. There are some parts in results but not very clear.}

\subsection{Hardware results}
We apply the proposed method on our experimental setup and we show the results
in Figure~\ref{fig:exp_result}, as well as the point cloud map for the
experimental run in Figure~\ref{fig:rover_explore}. In this experiment the robot
is tasked with exploring and mapping a set of corridors and hallways while
returning back to a charging spot. The map generated by the robot and the
trajectories taken by the robot during an experimental run are shown in
Figure~\ref{fig:rover_explore}.

%The robot then does a simple docking manoeuvre
%to get inside the charging region and starts a simulated charging sequence, in
%which a battery swap sequence is simulated to replenish the robot's
%energy.\todo{You just said it above, no need to repeat}

%In other words, the robot measures the power consumed through voltage and current readings, and as long as it moves at a certain speed it will consume energy at the same rate irrespective of the physical battery, and even after it finishes its charging sequence the rate of energy consumption will be independent of the physical battery. \todo[inline]{maybe needs some rewording}

\begin{figure*}[!htb]
	\centering
	\begin{subfigure}{0.49\textwidth}
		\centering
		\includegraphics[width=\linewidth]{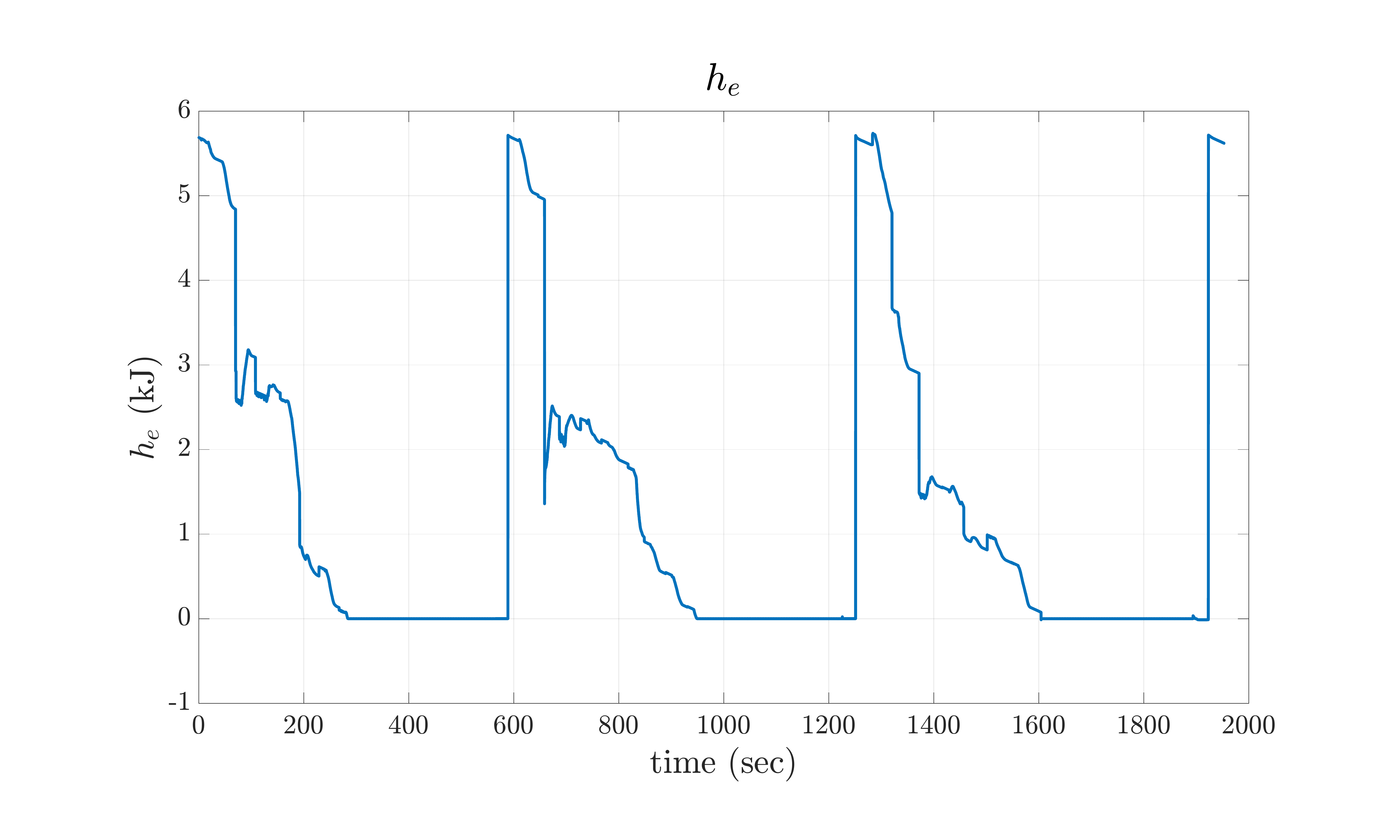}
		\caption{Energy sufficiency CBF ($h_e$)}
		\label{fig:rover_he}
	\end{subfigure}
	\begin{subfigure}{0.49\textwidth}
		\centering
		\includegraphics[width=\linewidth]{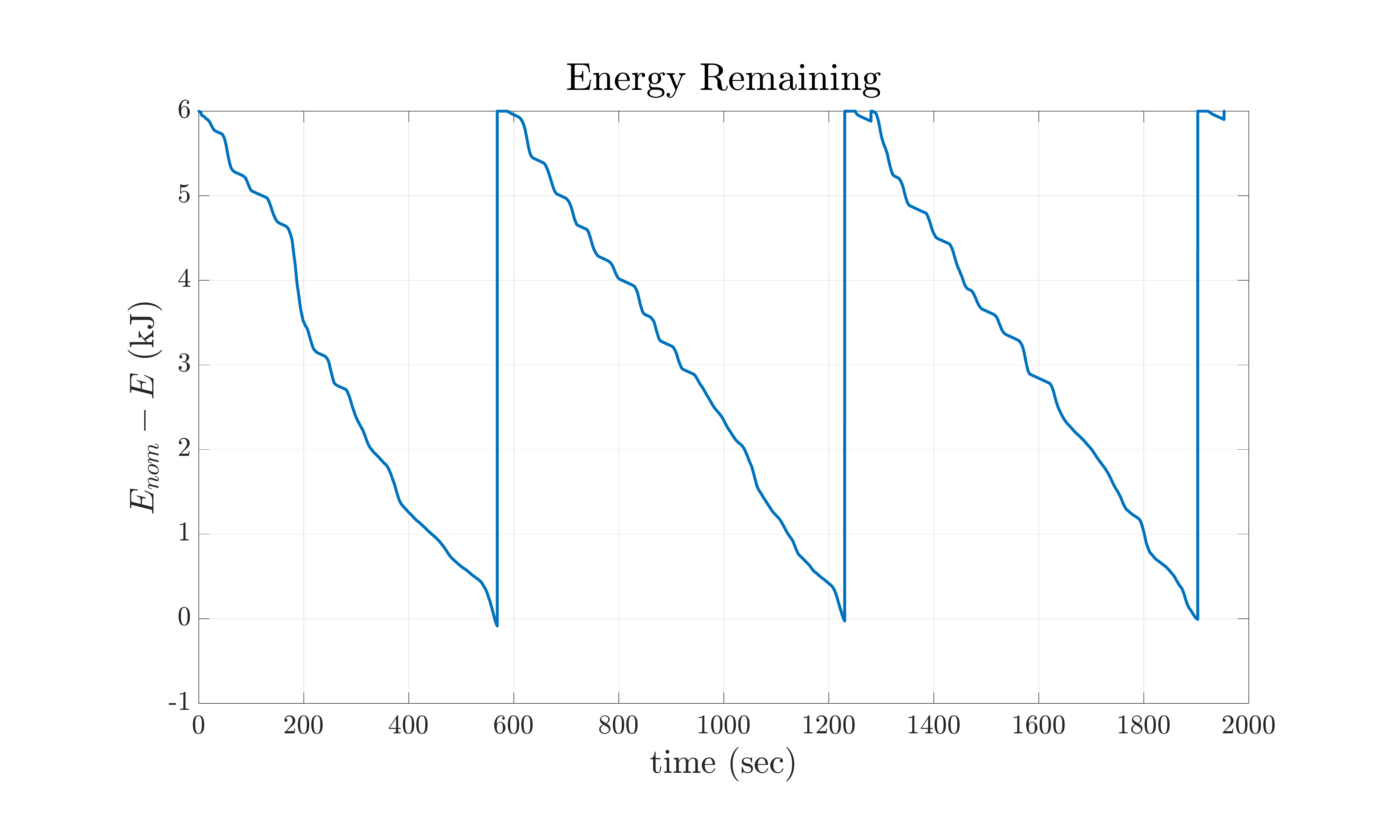}
		\caption{Remaining energy}
		\label{fig:rover_energy}
	\end{subfigure}
	\begin{subfigure}{0.49\textwidth}
		\centering
		\includegraphics[width=\linewidth]{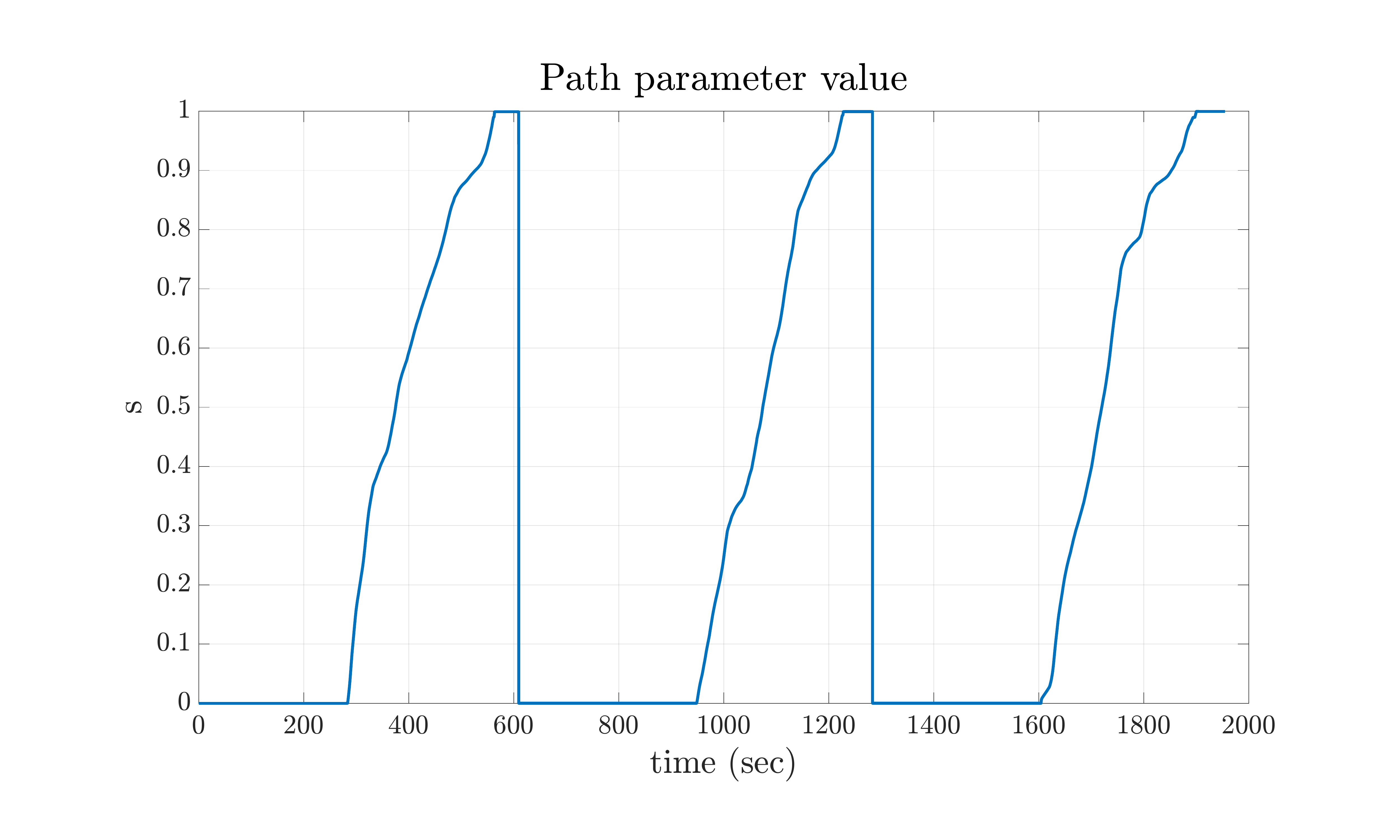}
		\caption{Path parameter}
		\label{fig:rover_s}
	\end{subfigure}
	\begin{subfigure}{0.49\textwidth}
		\centering
		\includegraphics[width=\linewidth]{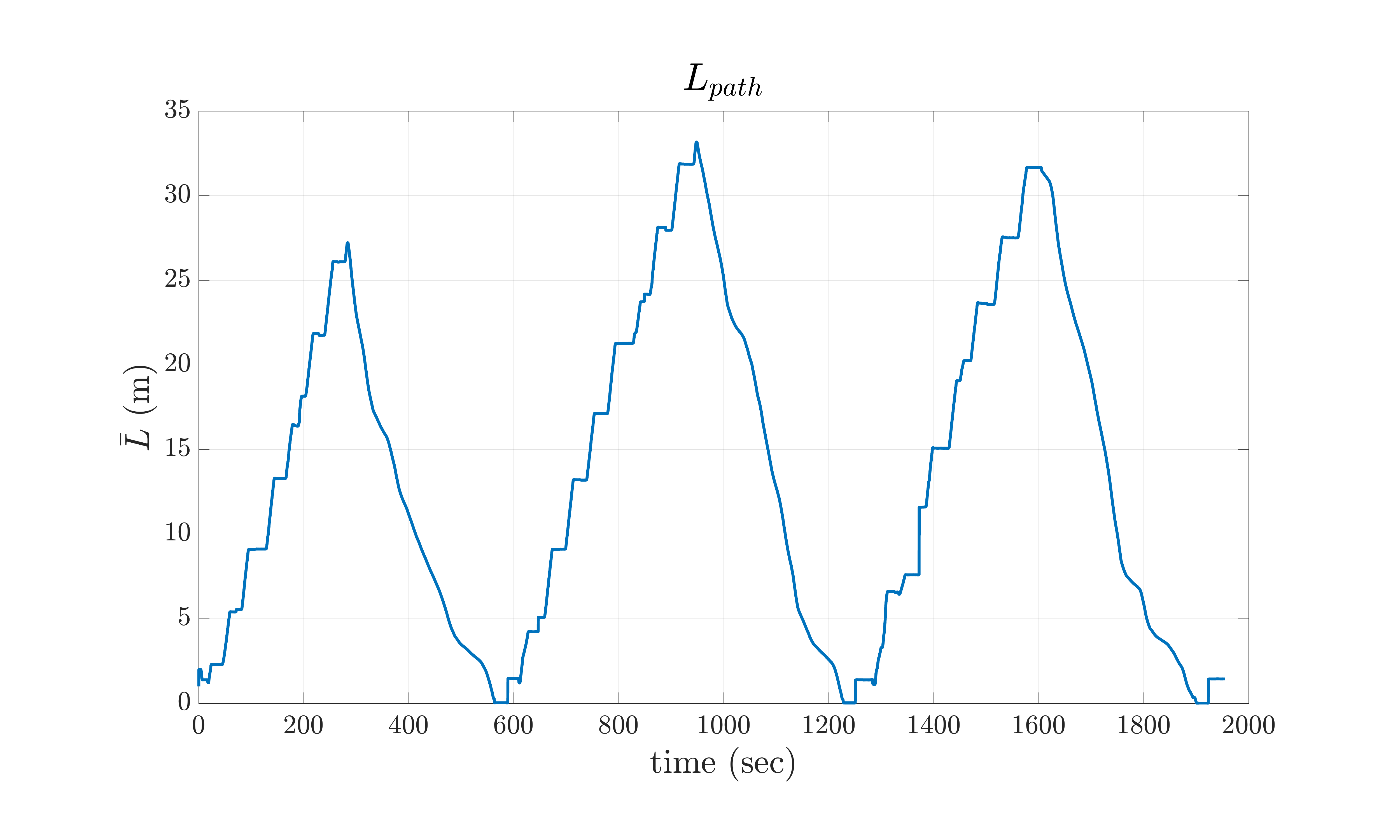}
		\caption{Path length}
		\label{fig:rover_L}
	\end{subfigure}
	\caption{Results from robot experiment. We emphasize the fact that the robot returns back to station with a depleted battery, meaning that our proposed framework is capable of maintaining energy sufficiency.}
	\label{fig:exp_result}
\end{figure*}

For this experiment the allocated energy budget is 7kJ and the desired return
speed was set to $v_r = 0.2$m/s. We note from Figure~\ref{fig:rover_energy} that
the robot consumes the energy budget fully by the time it arrives back to
recharge, which shows the ability of our proposed approach to maintain energy
sufficiency in cluttered environments. From Figure~\ref{fig:rover_s} the path
parameter value is equal to zero as long as the energy sufficiency constraint
\eqref{eqn:suff_constraint} is not violated, then when $h_e \approx 0$ (in
Figure~\ref{fig:rover_he}, indicating energy sufficiency being close to the
boundary of its safe set) it starts to increase and drive the robot back towards
the station along the path. Figure~\ref{fig:rover_explore} shows examples of
paths taken by the robot while exploring its environment.
%%% Local Variables:
%%% mode: latex
%%% TeX-master: "ijrr_draft"
%%% End:

%% file: conclusion.tex
In this work we present a CBF based method that provides guarantees on energy
sufficiency of a ground robot in an unknown and unstructured environment. Our
approach is to augment a sampling based path planner \cite[like
GBplanner,][]{dang2019graph} by a CBF layer, extending our work
\citep{fouad2022energy} to endow a robot with the ability to move along a path
in an energy aware manner such that the total energy consumed does not exceed a
predefined threshold. We described a continuous representation for piecewise
continuous paths produced by a path planner. We define a reference point that
slides along this continuous path depending on robot's energy. We
show the relationship between the constraints for controlling both the reference
point and robot's position and show conditions for these constraints to
complement each other. We demonstrate how these ideas are valid for dynamic
cases in which the path planner updates the path frequently and the robot is
carrying out a mission. Finally we demonstrate a method for adapting our
framework, based on a single integrator model, to a unicycle model. We highlight
through simulation and experimental results the ability of our method to deal
with unknown and unstructured environments while maintaining energy sufficiency.

Our proposed framework has the advantage of flexibility and adaptability to
different types of robot models and environments. Such framework can be useful
in many application where long term autonomy is needed, e.g. underground and
cave exploration, robot reinforcement learning, self driving cars in urban
environments, and many others.

As a future work we plan to extend our framework to be able to handle
coordination between multiple robots to share a charging station in the same
spirit as \citealt{fouad2022energy}, while being able to deal with
unstructured and complex environments. Another direction could be using online
estimation and learning techniques to handle power models that are variable by
nature and need constant adaptation, such as wind fields, snowy conditions, etc.
%%% Local Variables:
%%% mode: latex
%%% TeX-master: "ijrr_draft"
%%% End: